\documentclass[conference]{IEEEtran}
\IEEEoverridecommandlockouts
\newcommand\numbereq{\addtocounter{equation}{1}\tag{\theequation}}

\def\gB{{\mathcal{B}}}

\def\gD{{\mathcal{D}}}

\def\gK{{\mathcal{K}}}
\def\gL{{\mathcal{L}}}

\def\gO{{\mathcal{O}}}

\def\sR{{\mathbb{R}}}

\newcommand{\E}{\mathbb{E}}

\usepackage{cite}
\usepackage{amsmath,amssymb,amsfonts,amsthm}
\usepackage{algorithm}
\usepackage{graphicx}
\graphicspath{{./}{paper/}{figures/}}
\usepackage{textcomp}
\usepackage{xcolor}
\usepackage{booktabs}
\usepackage{multirow}
\usepackage{algpseudocode}
\usepackage{tcolorbox}
\usepackage{tablefootnote}
\newcommand{\alg}{\textbf{\texttt{HOSL}}}
\usepackage{array}
\usepackage{pdfpages}
\usepackage{pgfplots}
\usepackage{tikz}
\usetikzlibrary{decorations.pathreplacing, calligraphy, positioning}
\pgfplotsset{compat=1.18}

\def\BibTeX{{\rm B\kern-.05em{\sc i\kern-.025em b}\kern-.08em
    T\kern-.1667em\lower.7ex\hbox{E}\kern-.125emX}}
\begin{document}
\newtheorem{thm}{Theorem}[section]
\newtheorem{cor}[thm]{Corollary}
\newtheorem{lem}[thm]{Lemma}
\newtheorem{claim}{Claim}
\newtheorem{prop}[thm]{Proposition}
\newtheorem{ex}{Example}[section]
\theoremstyle{definition}
\newtheorem{defn}[thm]{Definition}
\newtheorem{assum}[thm]{Assumption}
\newtheorem{finalremark}{Final Remark}[section]
\newtheorem{rem}{Remark}
\newtheorem{sol}{Solution}

\title{HOSL: Hybrid-Order Split Learning for Memory-Constrained Edge Training
% {\footnotesize \textsuperscript{*}Note: Sub-titles are not captured in Xplore and should not be used}
}
\author{
\IEEEauthorblockN{
Aakriti Lnu\textsuperscript{1,*},
Zhe Li\textsuperscript{1,*},
Dandan Liang\textsuperscript{1,*},
Chao Huang\textsuperscript{2},
Rui Li\textsuperscript{1},
Haibo Yang\textsuperscript{1}
\thanks{*The first three authors contributed equally.}
}
\IEEEauthorblockA{
\textsuperscript{1}Dept. of Computing and Information Sciences Ph.D., Rochester Institute of Technology \\
\textsuperscript{2}School of Computing, Montclair State University
}
}
% \and
% \and
% \and
% \and
% \and
% \IEEEauthorblockN{Haibo yang}
% \IEEEauthorblockA{\textit{Dept. of Computing and Information Sciences Ph.D.} \\
% \textit{Rochester Institute of Technology)}\\
% Rochester, USA \\
% hbycis@rit.edu}
% }

\maketitle

\begin{abstract}
    Split learning (SL) enables collaborative training of large language models (LLMs) between resource-constrained edge devices and compute-rich servers by partitioning model computation across the network boundary. However, existing SL systems predominantly rely on first-order (FO) optimization, which requires clients to store intermediate quantities such as activations for backpropagation. This results in substantial memory overhead, largely negating benefits of model partitioning. In contrast, zeroth-order (ZO) optimization eliminates backpropagation and significantly reduces memory usage, but often suffers from slow convergence and degraded performance.
    In this work, we propose {\alg}, a novel \underline{H}ybrid-\underline{O}rder \underline{S}plit \underline{L}earning framework that addresses this fundamental trade-off between memory efficiency and optimization effectiveness by strategically integrating ZO optimization on the client side with FO optimization on the server side. By employing memory-efficient ZO gradient estimation at the client, {\alg} eliminates backpropagation and activation storage, reducing client memory consumption. Meanwhile, server-side FO optimization ensures fast convergence and competitive performance.
    Theoretically, we show that {\alg} achieves a $\mathcal{O}(\sqrt{d_c/TQ})$ rate, which depends on client-side model dimension $d_c$ rather than the full model dimension $d$, demonstrating that convergence improves as more computation is offloaded to the server.
    Extensive experiments on OPT models (125M and 1.3B parameters) across 6 tasks demonstrate that {\alg} reduces client GPU memory by up to 3.7$\times$ compared to the FO method while achieving accuracy within 0.20\%-4.23\% of this baseline. 
    Furthermore, {\alg} outperforms the ZO baseline by up to 15.55\%, validating the effectiveness of our hybrid strategy for memory-efficient training on edge devices.

\end{abstract}


\section{Introduction}

Split Learning (SL) is a collaborative training paradigm designed to enable efficient model training under strict data privacy and resource constraints \cite{vepakomma2018split, poirot2019split}.
In a typical SL architecture, a deep neural network is partitioned across a client and a server: the client executes the initial layers on its local data, while the server processes the remaining layers and computes the training loss.
Only intermediate activations and their associated gradients are exchanged between the client and the server, ensuring that raw data remains local to the client.
By offloading most of the computation to the server, SL enables scalable training in edge-cloud environments where edge devices operate under stringent memory and compute constraints \cite{lin2024efficient}.
Hence, SL has recently emerged as a promising paradigm for deploying and fine-tuning deep learning models on the edge, including large language models (LLMs)~\cite{gu2025vflair}.

Despite these architectural advantages, most existing SL systems rely on first-order (FO) optimization methods \cite{thapa2022splitfed, han2024convergence, radovivc2025towards}, which fundamentally limit their applicability in memory-constrained environments.
Although the model is partitioned between the client and the server, FO-based training still requires the client to store intermediate activations and construct the full computational graph for backpropagation across the split boundary.
This requirement induces a substantial memory footprint that scales with the depth and width of the client-side subnetwork, largely offsetting the benefits of model partitioning.
This issue becomes particularly pronounced when training extremely large models such as LLMs, where even a small prefix of the network can generate massive intermediate representations.
As a result, memory emerges as the primary bottleneck on the client side, preventing FO-based SL from being deployed on realistic, memory-constrained edge platforms.

Zeroth-order (ZO) optimization offers a fundamentally different training paradigm that is particularly appealing in such memory-constrained environments \cite{spall2002multivariate, ghadimi2013stochastic, nesterov2017random}.
Unlike FO methods, ZO algorithms estimate descent directions using only function evaluations, thereby completely eliminating the need for backpropagation and the storage of computational graphs.
This property makes ZO methods inherently lightweight in memory, as they do not require retaining intermediate activations or gradient tensors.
Recent advances have further improved the practicality of ZO methods for large-scale models.
In particular, Malladi et al.~\cite{malladi2023fine} proposed MeZO, a memory-efficient ZO optimizer that enables fine-tuning without storing activations or gradients.
This approach has since been adopted in a growing body of work \cite{li2025achieving, liu2025sparse, liang2025towards, li2025reconciling}.
However, ZO methods introduces significant estimation noise, leading to slower convergence and degraded performance compared to FO methods.
% However, the use of random perturbations in ZO methods introduces significant estimator variance, which typically leads to slower convergence and degraded performance compared to FO methods.
% As a result, ZO training often requires prohibitively many iterations to reach competitive accuracy.

These observations reveal a fundamental tension in SL system design: while FO methods deliver fast convergence, they impose prohibitive memory overhead on the client, whereas ZO methods offer significant memory savings at the cost of slower convergence.
This motivates the following question:

\begin{tcolorbox}[
  top=1pt,        
  bottom=1pt, 
  left=1.5pt, 
  right=1.5pt
]
\textit{Q: Can we design a split learning algorithm that preserves the memory efficiency of zeroth-order methods on the client, without sacrificing the convergence speed and final performance of first-order methods?}
\end{tcolorbox}

To address this challenge, we propose {\alg}, a hybrid-order optimization framework for SL that strategically integrates ZO optimization on the client side with FO optimization on the server side.
In our design, the client updates its local model parameters using ZO estimators, thereby avoiding backpropagation and eliminating the need to store intermediate activations or computational graphs.
Meanwhile, the server uses FO optimization to ensure rapid convergence and strong final performance.
This asymmetric optimization strategy preserves the modularity and privacy benefits of split learning, while directly addressing the client-side memory bottleneck.
As a result, our framework enables scalable and practical fine-tuning of LLMs in memory-constrained edge environments.

Our key contributions are summarized as follows: 
\begin{itemize}
    \item We propose {\alg}, a novel hybrid SL framework that strategically integrates ZO optimization on the client side with FO optimization on the server side. 
    More specifically, we employ in-place operation in ZO estimation at the client for significant memory saving and FO stochastic gradient descent (SGD) at the server to guarantee the overall rapid model convergence. 
    \item We provide a rigorous convergence analysis of {\alg} under nonconvex objectives with the rate of $\mathcal{O}(\sqrt{d_c/TQ})$, where $d_c$ is the client-side parameter dimension, $T$ is the number of iterations, and $Q$ is the number of perturbations. This rate depends only on the client-side model size rather than the full model dimension, highlighting the benefit of offloading more computation to the server.
    \item We validate the effectiveness of {\alg} through extensive fine-tuning experiments on OPT models spanning two scales (125M and 1.3B parameters) across six tasks. The results show that {\alg} reduces client-side memory consumption by up to 47\% compared to the fully FO-based baseline, while achieving competitive accuracy close to the FO baseline. 
\end{itemize}

\section{Related Work}

% \textbf{Split Learning (SL).}  SL \cite{vepakomma2018split} partitions models between clients and servers, providing a complementary method to enable LLM fine-tuning in environments with limited resources by offloading the primary training workload from clients to the central server. 
% SplitLoRA \cite{lin2024splitlora} combines split federated learning with LoRA adapters, allowing resource-constrained clients to handle only the initial transformer layers while offloading deeper layers and loss computation to a central server. 
% This framework reduces the client-side parameter footprint and communication overhead through low-rank adaptation. 
% However, SplitLoRA relies on standard backpropagation throughout the training pipeline: after the server computes gradients at the cut layer, these gradients are transmitted back to clients, which must then execute local backward passes. 
% This requirement means clients must still cache intermediate activations for gradient computation, imposing memory demands that scale with model depth and batch size. 
% Our work bridges these two paradigms by integrating zeroth-order optimization on the client side within a split learning framework, enabling clients to update their parameters through forward-pass perturbations alone while the server retains efficient first-order optimization. This hybrid approach inherits the client-side memory efficiency of ZO methods while leveraging server-side first-order updates to accelerate convergence compared to pure zeroth-order training.

\textbf{Split Learning (SL).}
SL~\cite{vepakomma2018split} was originally proposed to address the key limitations of federated learning (FL)~\cite{mcmahan2017communication} by splitting a model into several segments deployed unevenly across clients and the server. 
In this paradigm, only a lightweight portion of the model resides on the client side, which significantly reduces client-side memory consumption while enhancing data privacy by keeping raw data local~\cite{poirot2019split}. 
With the rapid growth of LLMs, SL has gained renewed attention for enabling the training of large-scale models on resource-constrained edge devices through model splitting~\cite{lin2024efficient, gu2025vflair, liang2025towards}. To further reduce the computational burden on edge devices, recent studies have integrated low-rank adaptation (LoRA) into SL by freezing a large portion of model parameters, thereby achieving improved training efficiency with an acceptable trade-off in accuracy~\cite{lin2024splitlora}. 
However, despite its advantages in reducing model size and computation, this approach still incurs considerable memory overhead due to the need to store gradient information from previous iterations. 
Such overhead can significantly exacerbate memory pressure on resource-constrained edge devices. 
These limitations highlight the need for more memory-efficient training mechanisms tailored to LLM deployment in edge environments.

\textbf{Zeroth-Order (ZO) Optimization.} 
ZO Optimization refers to a class of gradient-free methods that estimate gradient by only function value differences without requiring explicit gradient computation \cite{spall2002multivariate, ghadimi2013stochastic, nesterov2017random}. 
Recently, MeZO \cite{malladi2023fine} showed the potential of ZO optimization for memory-efficient fine-tuning of LLM. 
By perturbing the model parameters individually, MeZO reduces memory consumption to the inference level, achieving up to 7$\times$ memory savings compared to standard fine tuning with Adam on OPT-1.3B model. 
Given its significant memory efficiency, MeZO-style ZO optimization has been successfully applied in a variety of settings, such as FL \cite{fang2022communication, pmlr-v235-qin24a, li2025achieving}, split FL \cite{liang2025towards}, and efficient LLM fine-tuning \cite{guo2025zerothorder, liu2025sparse}. 
Nevertheless, a fundamental limitation of most existing ZO-based works is their slow convergence speed, which degrades training efficiency over a long training run. 
In this work, we utilize the memory efficiency of ZO optimization on the client side while mitigating its slow convergence by incorporating FO optimization at the server side. 
\algnewcommand{\LineComment}[1]{\State \textcolor{gray}{\(\triangleright\) #1}}
\algnewcommand{\SideComment}[1]{\hfill \textcolor{gray}{\(\triangleright\) #1}}
\newcommand*\circled[1]{\tikz[baseline=(char.base)]{
    \node[shape=circle,draw,inner sep=1pt] (char) {#1};}}
\section{Method}

\begin{figure}
    \centering
    \includegraphics[width=1\columnwidth]{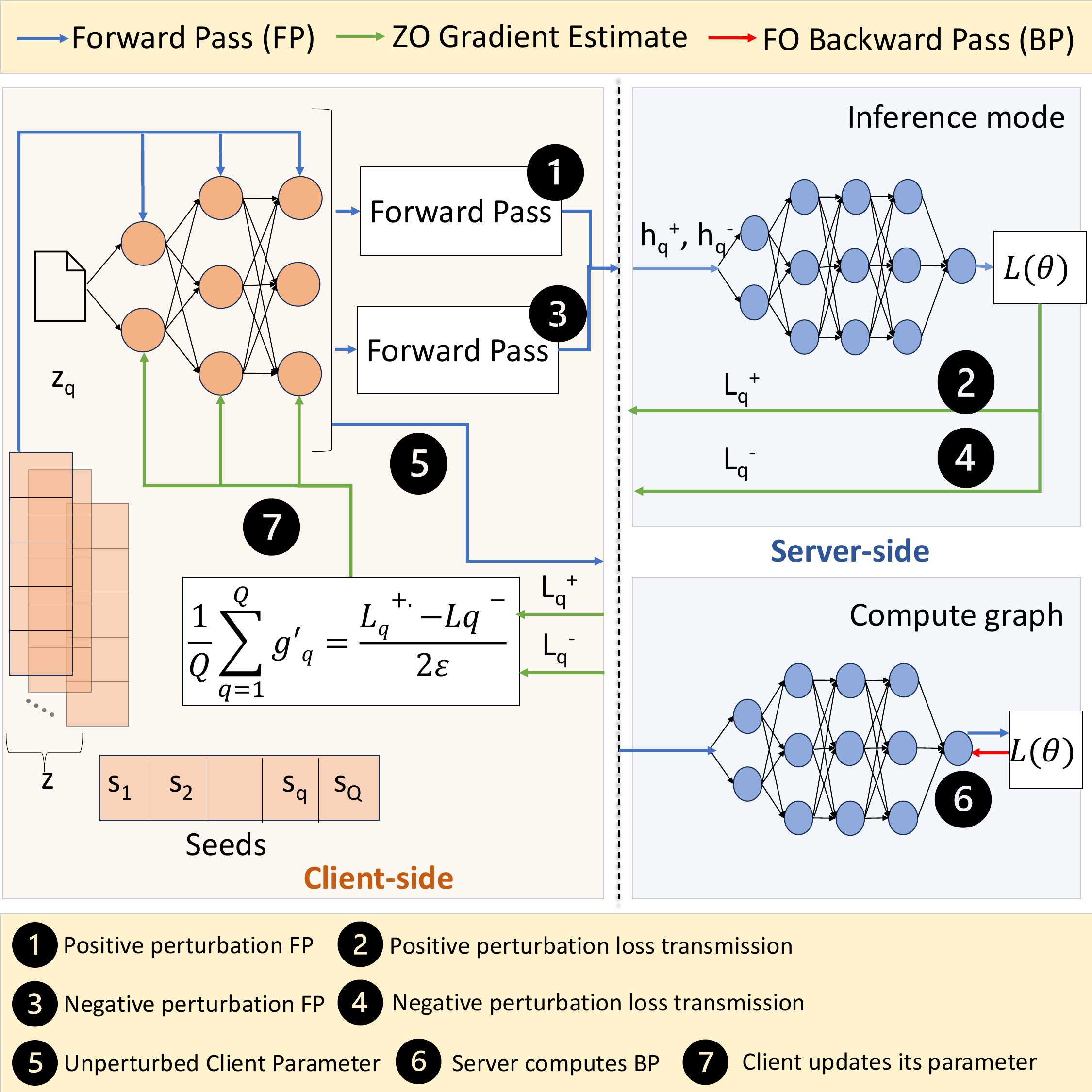}
    \caption{Overview of {\alg} Framework}
    \label{fig:hosl}
\end{figure}

\subsection{Split Learning}
We consider an SL system comprising a client and a server. The client holds a dataset $\mathcal{D} = \{\xi_i\}_{i=1}^{|\mathcal{D}|}$, where each sample $\xi_i\ = (x_i,y_i)$ consists of an input sequence $x_i$ and its label $y_i$. 
The entire model is partitioned at a designated cut layer \textit{k} into client-side $\textit{f}_c(\cdot;\theta_c)$ and server-side $\textit{f}_s(\cdot ;\theta_s)$ sub-models with parameters $\theta_c$ and $\theta_s$, respectively. We denote the complete model parameters as $\theta = \{\theta_c,\theta_s\}$.

The training objective is to minimize the expected loss over the training dataset $\mathcal{D}$:
\[
\min_{\theta_c, \theta_s} \; \mathcal{L}(\theta; \mathcal{D}) = \frac{1}{|\mathcal{D}|}\sum_{j \in \mathcal{D}}\ell(\textit{f}_s(\textit{f}_c(x_j;\theta_c);\theta_s),y_j). 
\]
% mini-batches $\mathcal{B} \subset \mathcal{D}$ of size \textit{B}:
% \[
% \min_{\theta_c, \theta_s} \; \mathcal{L}(\theta; \mathcal{B}) = \frac{1}{B}\sum_{j=1}^{B}\ell(\textit{f}_s(\textit{f}_c(x_j;\theta_c);\theta_s),y_j)
% \]
During the forward pass, the client computes intermediate activations $\textit{h}_j = \textit{f}_c(x_j;\theta_c)$ and transmits it to the server. The server completes the forward pass to compute loss.

\subsection{Optimization Methods} Our framework supports  two optimization methods that can be independently selected by each party: First-Order (FO) optimization via backpropagation and Zeroth-Order (ZO) optimization via gradient estimation.

\subsubsection{First-Order Optimization}
In FO optimization, stochastic gradients are computed via backpropagation. 
Within the SL framework, gradient computation requires coordination between the client and the server.

\textbf{Definition 1 (First-Order Gradient in SL).} 
Given the loss $\mathcal{L}(\theta; \mathcal{B})$ on a batch of data $\mathcal{B} \subset \mathcal{D}$ computed at the server, the gradients are obtained as follows. The server computes its gradient directly:
\begin{equation}
\nabla_{\theta_s} \mathcal{L}(\theta; \mathcal{B}) = \frac{\partial \mathcal{L}}{\partial \theta_s}.
\label{eq:server_grad}
\end{equation}
The server also computes and transmits the gradient with respect to activations at the cut layer:
\begin{equation}
g_h = \nabla_h \mathcal{L} = \frac{\partial \mathcal{L}}{\partial h}.
\label{eq:activation_grad}
\end{equation}
The client then completes backpropagation through its sub-model by chain rule:
\begin{equation}
\nabla_{\theta_c} \mathcal{L}(\theta; \mathcal{B}) = \left(\frac{\partial h}{\partial \theta_c}\right)^\top g_h.
\label{eq:client_grad}
\end{equation}

\textbf{Definition 2 (First-Order Parameter Update).} Parameters are updated via stochastic gradient descent (SGD):
\begin{align}
\theta_c^{t+1} &= \theta_c^t - \eta_c \nabla_{\theta_c} \mathcal{L}(\theta^t; \mathcal{B}_t) \label{eq:fo_client_update}, \\
\theta_s^{t+1} &= \theta_s^t - \eta_s \nabla_{\theta_s} \mathcal{L}(\theta^t; \mathcal{B}_t),  \label{eq:fo_server_update}
\end{align}
where $\eta_c$ and $\eta_s$ denote the learning rates at the client and server side, respectively. 
FO optimization requires storing intermediate activations during the forward pass, incurring memory cost proportional to sequence length and batch size.

\subsubsection{Zeroth-Order Optimization}
ZO optimization estimates gradients using only forward passes and hence eliminates the need for backpropagation and the associated memory overhead of storing activations. 
We adopt the gradient estimator from MeZO.

\textbf{Definition 3 (Zeroth-Order Gradient Estimator).} Given parameters $\theta \in \mathbb{R}^d$, loss function $\mathcal{L}$, and mini-batch $\mathcal{B}$, the gradient estimate is:
\begin{equation}
\widetilde{\nabla} \mathcal{L}(\theta; \mathcal{B}) = \frac{\mathcal{L}(\theta + \epsilon z; \mathcal{B}) - \mathcal{L}(\theta - \epsilon z; \mathcal{B})}{2\epsilon} \cdot z,
\label{eq:spsa_gradient}
\end{equation}
where $z \in \mathbb{R}^d$ is a random perturbation vector with components $z_j {\sim} \mathcal{N}(0, 1)$ for $j = 1, \ldots, d$, and $\epsilon > 0$ is the perturbation scale.

% \textcolor{blue}{Dandan: I would suggest different notations for \eqref{eq:spsa_gradient} and \eqref{eq:multi_pert_gradient}.}
%
For variance reduction, $Q$ independent perturbations can be averaged as follows:
% \begin{equation}
% \widetilde{\nabla} \mathcal{L}(\theta; \mathcal{B}) = \frac{1}{Q} \sum_{q=1}^{Q} \frac{\mathcal{L}(\theta + \epsilon z^{(q)}; \mathcal{B}) - \mathcal{L}(\theta - \epsilon z^{(q)}; \mathcal{B})}{2\epsilon} \cdot z^{(q)}
% \label{eq:multi_pert_gradient}
% \end{equation}
\begin{equation}
\hat{g}_Q = \frac{1}{Q} \sum_{q=1}^{Q} \frac{\mathcal{L}(\theta + \epsilon z^{(q)}; \mathcal{B}) - \mathcal{L}(\theta - \epsilon z^{(q)}; \mathcal{B})}{2\epsilon} \cdot z^{(q)},
\label{eq:multi_pert_gradient}
\end{equation}
where $\{z^{(q)}\}_{q=1}^Q$ are independent samples, each with $z^{(q)} \in \mathbb{R}^d$ and $z^{(q)}_j {\sim} \mathcal{N}(0, 1)$ for all $j$ and $q$. This requires $2Q$ forward passes per update step.

\textbf{Definition 4 (Zeroth-Order SGD).} The ZO SGD update rule mirrors FO SGD but substitutes the estimated gradient:
\begin{equation}
\theta^{t+1} = \theta^t - \eta \widetilde{\nabla} \mathcal{L}(\theta^t; \mathcal{B}_t).
\label{eq:zo_update}
\end{equation}

\subsection{Proposed Algorithm: Mixed-Optimizer Split Learning}
We now present the training procedure for our {\alg}, where the client uses ZO optimization and the server uses FO SGD.

\subsubsection{Phase 0: Initialization} 
The server initializes the model and partitions it in a designated cut layer. 
The first $k$ layers form the client module, and the remaining layers form the server module. 
Both parties initialize their trainable parameters based on the fine-tuning method (full fine-tuning, LoRA, or prefix tuning).
% The server initializes an SGD optimizer, while the client requires no optimizer since it uses ZO optimization.

\subsubsection{Phase 1: ZO Gradient Estimation} 
% For each perturbation $q = 1, \ldots, Q$:
For each perturbation $q = 1, \ldots, Q$ (illustrated in Figure~\ref{fig:hosl}, steps \circled{1}--\circled{4}):
\paragraph{Step 1.1: Positive Perturbation Forward Pass} 
The client perturbs its parameters by $+\epsilon z_q$ and computes its forward pass (step \circled{1}), sending the resulting activations $h_q^+$ to the server. 
The server operates in \textit{inference mode}: it completes the forward pass and computes the loss $\mathcal{L}^+_q$ without computing or storing gradients, returning only the scalar loss to the client (step \circled{2}).

\paragraph{Step 1.2: Negative Perturbation Forward Pass} 

The client perturbs its parameters to $-\epsilon z_q$ from the original position (by applying $-2\epsilon z_q$ to the current state) and computes its forward pass (step \circled{3}), sending the activations $h_q^-$ to the server. 
The server again operates in inference mode, computing the loss $\mathcal{L}^-_q$ without backpropagation and returning the scalar loss (step \circled{4}). 
Crucially, the server parameters $\theta_s$ remain unchanged between $\mathcal{L}^+_q$ and $\mathcal{L}^-_q$, ensuring an unbiased ZO gradient estimate.
\paragraph{Step 1.3: Restore and Store Gradient} 
The client restores its parameters to the original state and computes the ZO gradient estimate $\hat{g}_q$ using Equation~\ref{eq:multi_pert_gradient}, as shown in the gradient computation box in Figure~\ref{fig:hosl}. 
The client \textit{stores} this gradient estimate along with the seed $s_q$ but does not update its parameters yet.
\subsubsection{Phase 2: Server Update} 
After completing all $Q$ perturbations, the client performs a forward pass with its \textit{original} (unmodified) parameters (step \circled{5}) and sends the activations to the server with a \texttt{compute\_grad} signal.
The server computes the forward pass with gradient tracking enabled, performs backpropagation to obtain $\nabla_{\theta_s}\mathcal{L}$, and updates its parameters using SGD (step \circled{6}).

\subsubsection{Phase 3: Client Update} 
After receiving acknowledgment from the server, the client applies all $Q$ stored ZO gradient estimates to update its parameters (step \circled{7}).
After $T$ training iterations, the fine-tuned model parameters on both the client and the server represent the trained model.
\begin{algorithm}[t]
\caption{Hybrid-Order Split Learning (Client-Side)}
\label{alg:client}
\begin{algorithmic}[1]
\Require Client parameters $\theta_c$, learning rate $\eta_c$, perturbation scale $\epsilon$, number of perturbations $Q$, training iterations $T$
\State \textbf{Initialize:} Connect to server
\For{$t = 1, 2, \ldots, T$}
    \State Sample mini-batch $(x, y)$ from local data
    \State Initialize gradient accumulator: $\{(\hat{g}_q, s_q)\}_{q=1}^Q \gets \emptyset$
    \State
    \LineComment{\textbf{Phase 1:} ZO gradient estimation with $Q$ perturbations}
    \For{$q = 1, 2, \ldots, Q$}
        \State Sample random seed $s_q$
        \State
        \LineComment{Step 1.1: Positive perturbation}
        \State $\theta_c \gets \textsc{Perturb}(\theta_c, +\epsilon, s_q)$
        \State $h^+_q \gets f_c(x; \theta_c)$
        \State Send $(h^+_q, y, \texttt{inference})$ to server
        \State Receive $\mathcal{L}^+_q$ from server
        \State
        \LineComment{Step 1.2: Negative perturbation}
        \State $\theta_c \gets \textsc{Perturb}(\theta_c, -2\epsilon, s_q)$ \SideComment{Now at $\theta_c - \epsilon z_q$}
        \State $h^-_q \gets f_c(x; \theta_c)$
        \State Send $(h^-_q, y, \texttt{inference})$ to server
        \State Receive $\mathcal{L}^-_q$ from server
        \State
        \LineComment{Step 1.3: Restore and compute gradient (do not update yet)}
        \State $\theta_c \gets \textsc{Perturb}(\theta_c, +\epsilon, s_q)$ \SideComment{Restore to original}
        \State $\hat{g}_q \gets \dfrac{\mathcal{L}^+_q - \mathcal{L}^-_q}{2\epsilon Q}$ \SideComment{Store projected gradient}
    \EndFor
    \State
    \LineComment{\textbf{Phase 2:} Server FO update (before client update)}
    \State $h \gets f_c(x; \theta_c)$ \SideComment{Forward with original $\theta_c$}
    \State Send $(h, y, \texttt{compute\_grad})$ to server
    \State Wait for server update acknowledgment
    \State
    \LineComment{\textbf{Phase 3:} Client ZO update (after server update)}
    \For{$q = 1, 2, \ldots, Q$}
        \State $\theta_c \gets \textsc{ZOUpdate}(\theta_c, \hat{g}_q, \eta_c, s_q)$
    \EndFor
\EndFor
\end{algorithmic}
\end{algorithm}

% ============================================================================
% ALGORITHM 2: SERVER-SIDE (FO)
% ============================================================================
\begin{algorithm}[t]
\caption{Hybrid-Order Split Learning (Server-Side)}
\label{alg:server}
\begin{algorithmic}[1]
\Require Server parameters $\theta_s$, learning rate $\eta_s$
\State \textbf{Initialize:} Wait for client connection
\While{training in progress}
    \State Receive $(h, y, \texttt{phase})$ from client
    \State
    \If{$\texttt{phase} = \texttt{inference}$}
        \LineComment{\textbf{Phase 1:} Inference only during client ZO estimation}
        \State \textbf{with} inference mode: \SideComment{No gradient computation}
        \State \hspace{1em} $\hat{y} \gets f_s(h; \theta_s)$
        \State \hspace{1em} $\mathcal{L} \gets \textsc{Loss}(\hat{y}, y)$
        \State Send $\mathcal{L}$ to client
        \State \textcolor{gray}{\(\triangleright\) $\theta_s$ remains unchanged}
        \State
    \ElsIf{$\texttt{phase} = \texttt{compute\_grad}$}
        \LineComment{\textbf{Phase 2:} Compute gradients and update}
        \State $\hat{y} \gets f_s(h; \theta_s)$
        \State $\mathcal{L} \gets \textsc{Loss}(\hat{y}, y)$
        \State Compute $\nabla_{\theta_s} \mathcal{L}$ via backpropagation
        \State $\theta_s \gets \theta_s - \eta_s \cdot \nabla_{\theta_s} \mathcal{L}$ \SideComment{FO update}
        \State Send acknowledgment to client \SideComment{No loss needed}
    \EndIf
\EndWhile
\end{algorithmic}
\end{algorithm}

% ============================================================================
% ALGORITHM 3: SUBROUTINES
% ============================================================================
\begin{algorithm}[t]
\caption{Subroutines for Hybrid-Order Split Learning}
\label{alg:subroutines}
\begin{algorithmic}[1]
\Function{Perturb}{$\theta, \delta, s$}
    \LineComment{Perturb parameters by $\delta \cdot z$ where $z \sim \mathcal{N}(0, I)$}
    \State Set random seed to $s$
    \For{each parameter $\theta^{(i)} \in \theta$}
        \State $z^{(i)} \sim \mathcal{N}(0, I)$ \SideComment{Same shape as $\theta^{(i)}$}
        \State $\theta^{(i)} \gets \theta^{(i)} + \delta \cdot z^{(i)}$
    \EndFor
    \State \Return $\theta$
\EndFunction
\State
\Function{ZOUpdate}{$\theta, \hat{g}, \eta, s$}
    \LineComment{Update parameters using projected gradient and regenerated perturbation}
    \State Set random seed to $s$ \SideComment{Regenerate same $z$ as \textsc{Perturb}}
    \For{each parameter $\theta^{(i)} \in \theta$}
        \State $z^{(i)} \sim \mathcal{N}(0, I)$
        \State $\theta^{(i)} \gets \theta^{(i)} - \eta \cdot \hat{g} \cdot z^{(i)}$
    \EndFor
    \State \Return $\theta$
\EndFunction
\end{algorithmic}
\end{algorithm}

\section{Convergence Analysis}
In this section, we present a rigorous convergence analysis of our {\alg} algorithm. 
We first state our assumptions as follows.

\begin{assum}[$L$-Smooth]\label{main_assum:smooth}
    The loss function $\gL$ is $L$-smooth for $\forall \theta_1,\theta_2\in \sR^d$:
    \[\|\ \nabla \gL(\theta_1)- \nabla \gL(\theta_2)\|\le L\|\ \theta_1-\theta_2 \|\ ,\]
    thus
    \[ \gL(\theta_2) \le \gL(\theta_1) + \langle\nabla \gL(\theta_1),\theta_2-\theta_1\rangle +\frac{L}{2} \|\ \theta_2-\theta_1 \|^2. \]
\end{assum}
\begin{assum}[Unbiased Gradient and Bounded variance]\label{main_assum:variance}
    The variance of the stochastic gradients w.r.t. the client and server are upper bounded respectively. For $\forall \theta \in \sR^d$, let $\theta=[\theta_c, \theta_s]$, where $\theta_c\in \sR^{d_c}, \theta_s\in \sR^{d_s}$ represent the parameters on client and server, respectively. Then for any mini-batch $\gB \in \gD$:
    \[
    \|\nabla_{\theta_c} \gL(\theta;\gB)-\nabla_{\theta_c} \gL(\theta)\|^2\le \sigma_{c}^2,
    \]
    \[
    \|\nabla_{\theta_s} \gL(\theta;\gB)-\nabla_{\theta_s} \gL(\theta)\|^2\le \sigma_{s}^2.
    \]   
\end{assum}

These first two assumptions are standard in non-convex optimization~\cite{ghadimi2013stochastic} and split learning~\cite{liang2025towards}.

\subsection{Convergence Analysis for {\alg}}

In this subsection, we first analyze the convergence rate of our {\alg}, for which we have the following result:
\begin{thm}\label{thm:main_algorithm}
    Under Assumption \ref{main_assum:smooth} and \ref{main_assum:variance}. If the learning rates on client and server satisfy $\eta_s \leq \frac{3}{4L}, 
    \eta_c \leq \frac{Q}{4Ld_c}$, the sequence of iterates generated by our algorithm satisfies:
    \vspace{-1pt}
    \begin{align*}
    \textstyle\frac{1}{T}\sum_{t=0}^{T}\E[\|\nabla_{\theta}\gL(\theta^t)\|^2]
    \le \frac{4\mathcal{F}}{\eta T} +2L\eta\sigma_s^2+\frac{4\eta L d_c\sigma_{c}^2 }{Q}+\frac{L^2\lambda^2 d_c^3}{Q}
   \end{align*}

\end{thm}
Here, $\mathcal{F}=\E[\gL(\theta^{0})-\gL(\theta^{T})]$, and $d_c$ denotes the parameter dimensionality of the client-side model. The parameter $\lambda$ corresponds to the smoothing parameter of the ZO oracle defined in \eqref{eq:spsa_gradient}, while $\sigma_c^2$ and $\sigma_s^2$ represent the upper bounds on the stochastic gradient variances at the client and server sides, respectively.

The first term, $\frac{4\mathcal{F}}{\eta T}$, characterizes the optimization error due to initialization and decays at a rate of $\gO(1/T)$, which is consistent with the standard convergence behavior of stochastic gradient descent. The second and third terms quantify the error introduced by the variance of the stochastic gradient
estimates on the server and client, respectively. In particular, the server-side variance term decrease as the learning rate $\eta$ decreases, indicating that the effect of stochastic error can be mitigated by using smaller step sizes. On the client side, the variance term exhibits a dependence on the model dimensionality $d_c$, consistent with known observation for ZO optimization in other literature. Importantly, this term decreases linearly with the number of perturbations $Q$, reflecting the variance reduction effect of multi-point gradient estimation. This suggests that allocating fewer model parameters to the client and increasing the number of perturbations can effectively control stochastic error on client side, thereby accelerating convergence.

The final term arises from the bias introduced by the ZO oracle. This term is independent of the learning rate and decreases with a smaller smoothing parameter $\lambda$. Moreover, it also diminishes as the number of perturbations $Q$ increases, since averaging over more perturbations improves the accuracy of the gradient estimator. 

\begin{cor}\label{cor:main_algorithm}
 Based on Theorem \ref{thm:main_algorithm}, let the unified learning rate satisfies $\eta:= \eta_c=\eta_s=\frac{\sqrt{Q}}{ \sqrt{d_c TL}}$; let the smoothing parameter satisfies $\lambda^2 \le\frac{\sqrt{Q}}{\sqrt{d_c^5T L^3}}$. Then we have the following convergence rate:
\begin{align*}
    &\frac{1}{T}\sum_{t=0}^{T}\E[\|\nabla_{\theta}\gL(\theta^t)\|^2]
    \le \frac{4\sqrt{d_c L}}{\sqrt{TQ}}\mathcal{F}+\frac{2\sqrt{LQ}}{\sqrt{d_c T}}\sigma_s^2\\
    &+\frac{4\sqrt{d_cL}\sigma_{c}^2 }{\sqrt{TQ}}+\frac{\sqrt{d_cL}}{\sqrt{TQ}}\numbereq
    \label{eq:col}
\end{align*}
\end{cor}
\begin{rem} Under the parameter choices specified in Corollary~\ref{cor:main_algorithm}, all dominant terms in equation \eqref{eq:col} converge at the rate of $\gO(\sqrt{d_c/TQ})$. This result highlights the role of multi-perturbation ZO estimation in accelerating convergence. In particular, increasing the number of perturbations $Q$ effectively reduces the stochastic error on the client side, thereby facilitating overall convergence performance. Moreover, this result also reveals that the convergence rate can be further improved by putting less portion of the model on the client side, that's being said with the convergence can become faster with the decrease of $d_c$. 

This observation can be attributed to the variance-reduction and bias-reduction properties of multi-point zeroth-order gradient estimators. As established in Theorem~\ref{thm:main_algorithm}, a larger $Q$ mitigates both the stochastic variance and the estimation bias introduced by the ZO optimization. Consequently, the step size can be scaled as $\eta=\gO(\sqrt{Q})$ without compromising stability, leading to an effective linear speedup with respect to $Q$.
In contrast, the second term which captures the accumulation of stochastic error on the server side, converges at a rate of $\gO(Q/d_cT)$. Since this term depends on FO gradient updates at the server, it is independent of the ZO estimation error and therefore does not benefit from increasing $Q$. Notably, because the client-side parameter dimension $d_c$ is typically comparable to or larger than $Q$, this term decays faster than the ZO affected terms and does not dominate the overall convergence rate.
\end{rem}

\begin{rem}
The convergence rate can be further improved by reducing the client-side model dimension $d_c$. This observation provides an important guideline for model-splitting strategies: to achieve better convergence performance, a smaller fraction of the model should be deployed on the client. Intuitively, FO gradient descent is inherently more efficient than ZO gradient descent. Therefore, assigning more computationally complex components to the server leads to faster convergence. Overall, this analysis demonstrates the advantage of the proposed mixed FO–ZO update scheme, which leverages the efficiency of FO optimization on the server while retaining the memory-efficient property of ZO optimization on the client. By combining FO and ZO methods, the algorithm achieves improved convergence speed while effectively accommodating stricter resource constraints on the client side.

\end{rem}

\begin{table}[hb!]\label{table_convergence rate}
    \centering
    \caption{Comparison of convergence rate with FO/ZO SGD paradigm}
    \begin{tabular}{l|c|c|c}
    \toprule
    \textbf{Methods} & \textbf{FO SGD} & \textbf{ZO SGD} &\textbf{{\alg}} \\
    \midrule
    \textbf{Convergence Rate} & $\gO(\sqrt{1/T})$ & $\gO(\sqrt{d/TQ}))$ &$\gO(\sqrt{d_c/TQ})$\\
     \bottomrule
    \end{tabular}
    \label{tab:convergence rate}
\end{table}

\begin{rem}
Table \ref{tab:convergence rate} compares the convergence rate of the proposed mix-optimizer algorithm with two baselines: FO and ZO SGD paradigm. This comparison validates the theoretical advantage of the proposed method. 
Among the three approaches, FO SGD achieves the fastest convergence rate, as it directly exploits exact gradient information for parameter updates. In contrast, ZO SGD relies on noisy gradient estimators constructed from function evaluations, which introduces additional bias and variance. As a result, ZO methods typically require smaller step sizes to ensure stability, leading to slower convergence.

This limitation is reflected in the convergence rate of ZO SGD which exhibits a linear dependence on the model dimension $d$. Although increasing the number of perturbations $Q$ can partially mitigate the stochastic error, the dimensional dependence remains dominant and significantly degrades convergence in high-dimensional settings.

The proposed mixture optimizer achieves a convergence rate that lies between those of FO and ZO SGD. Compared to ZO optimization, the key improvement is the substantially weaker dependence on the model dimension: the convergence rate depends only on the client-side model size $d_c$ rather than full model dimension $d$. Consequently, the convergence performance can be further improved by allocating a smaller portion of the model to the client.

This acceleration stems from the incorporation of FO updates on the server side, where more complex components of the model are optimized efficiently using exact gradients. By offloading computationally intensive updates to the server, the proposed method effectively combines the efficiency of FO optimization while the client side enjoys the benefits of ZOO. Moreover, the algorithm achieves a linear speedup with respect to the number of perturbations $Q$, matching the typical ZO SGD and indicating that our algorithm can achieve the tightest bound under standard assumptions.  
\end{rem}

\section{Experiments}
\subsection{Experiment Setup}
\textbf{Models \& Training Configurations.} 
We evaluated our approach on the OPT family\cite{zhang2022opt} (125M and 1.3B) under both full-parameter and parameter-efficient fine-tuning (e.g., LoRA \cite{hu2022lora}). 
Following the prompt templates from \cite{malladi2023fine}, we benchmarked three optimizer configurations: 
(1) ZO-ZO: both client and server sides use ZO SGD, 
(2) FO-FO: both client and server sides use FO SGD, and
(3) ZO-FO (Ours): the client uses ZO SGD, and the server employs FO SGD. We fix the number of $Q = 10$ perturbation vectors across all experiments. To further investigate the impact of $Q$, we additionally evaluate $Q \in \{1, 5, 10\}$. The results confirm our theoretical insight in Corollary \ref{cor:main_algorithm}, demonstrating that larger $Q$ consistently leads to improved accuracy and accelerated convergence.

\textbf{Tasks \& Hardware.} 
We assessed classification performance on a subset of GLUE \cite{wang2018glue} and SuperGLUE \cite{wang2019superglue} tasks, including SST-2 \cite{socher2013recursive}, CB \cite{de2019commitmentbank}, WIC \cite{pilehvar2019wic}, WSC \cite{kocijan2020review}, BoolQ \cite{clark2019boolq}, and RTE \cite{bowman2015large}. 
All experiments were conducted on an NVIDIA A$100$ GPU with $40$GB memory. Detailed hyperparameter configurations are provided in 
% Appendix \ref{appendix: exp} of our technical report \cite{ExtendedVersion}.
Appendix~VII-A.
% of our technical report \cite{ExtendedVersion}.

\textbf{Model Partitioning.} 
We choose the split layer at $5$ for both OPT-125M and OPT-1.3B, assigning the first five layers to the client and the remaining layers to the server. To further investigate the impact of the split location, we additionally evaluate $k \in \{3, 5, 7\}$ for OPT-1.3B. Detailed configurations and corresponding results are provided in Appendix~VII-C.
% \ref{appendix: abl}.
% of our technical report \cite{ExtendedVersion}.

% Appendix~C of our technical report \cite{ExtendedVersion}.

\textbf{Memory Measurement.} We report peak GPU memory consumption for each model and dataset configuration. These empirical measurements align with our analytical memory estimates, which account for model parameters, gradients, activations, and perturbation vectors. Detailed derivations are provided in Appendix~VII-B.
% \ref{appendix: memory}. 
% of our technical report \cite{ExtendedVersion}.
\subsection{Experiment Results}

\textbf{{\alg} Reduces Client-Side GPU Memory by up to 3.7$\times$ Compared to FO-FO.} 
As illustrated in Figure~\ref{fig:cgpu_comparison}, our method maintains client GPU memory consumption equivalent to ZO-ZO while achieving significant memory reduction compared to FO-FO. For OPT-125M under full-parameter fine-tuning, our approach reduces CGPU usage by 1.1$\times$ to 3.7$\times$ across tasks, with the largest savings on BoolQ (2.36~GB vs.\ 8.67~GB). For the larger OPT-1.3B model, memory reductions range from 1.2$\times$ to 2.1$\times$, with BoolQ again showing maximum savings (7.23~GB vs.\ 15.41~GB).

\textbf{{\alg} Achieves up to 15.55\% Higher Accuracy Than ZO-ZO.}
Tables~\ref{table:FPT-results} and~\ref{table:LoRA-results} present test accuracy comparisons for full-parameter and LoRA-based fine-tuning, respectively. Our hybrid ZO-FO approach consistently outperforms the ZO-ZO baseline across all tasks and model configurations. For full-parameter fine-tuning, accuracy improvements over ZO-ZO range from 0.69\% on BoolQ (OPT-125M) to 15.55\% on RTE (OPT-1.3B). LoRA-based fine-tuning yields improvements from 0.60\% to 13.16\% over ZO-ZO. 
Compared to the FO-FO baseline, our method incurs only modest accuracy reductions: 0.41\%-4.23\% for full-parameter fine-tuning and 0.20\%-2.05\% for LoRA.

% \textbf{{\alg} Achieves the Best Accuracy-Memory Trade-off in LLM Fine-Tuning Tasks.} 
% Table~\ref{table:FPT-results} presents a comprehensive comparison of test accuracy and GPU memory consumption across various tasks and models. 
% As observed, our proposed method achieves a superior trade-off between performance and efficiency. Specifically, our approach delivers test accuracy comparable to the fully first-order baseline (FO-FO) and significantly outperforms the pure zeroth-order method (ZO-ZO). Crucially, this performance gain is achieved while maintaining ultra-low Client GPU (CGPU) usage, identical to that of the ZO-ZO baseline, thereby effectively enabling high-performance fine-tuning on resource-constrained edge devices. 

\textbf{{\alg} Achieves the Best Accuracy-Memory Trade-off for LLM Fine-Tuning.} 
The results demonstrate that our method occupies a favorable position in the accuracy-memory trade-off space. By partitioning optimization responsibilities - ZO on the client with FO on the server, {\alg} achieves test accuracy within 0.41\%-4.23\% of FO-FO while reducing client memory to the level of ZO-ZO. This design enables high-performance fine-tuning on resource-constrained edge devices without sacrificing the convergence benefits of gradient-based optimization on the server side. Additional ablation studies over split positions ($k \in {3, 5, 7}$), wall-clock training time, and communication cost (Appendix VII-C, Tables VII – IX) 
% of our technical report \cite{ExtendedVersion}) 
 further confirm that these advantages hold consistently across configurations.

\begin{table}[htbp]
\centering
\caption{Test Accuracy and GPU Memory Comparison of OPT Models (\textbf{Full-Parameter} Fine-Tuning)}
\label{table:FPT-results}
\begin{tabular}{c|c|c|ccc}
\toprule
\textbf{Model} & \textbf{Task} & \textbf{Method} & \textbf{Test Acc} & \textbf{CGPU} \tablefootnote{CGPU: Peak GPU memory at the client side} & \textbf{SGPU} \tablefootnote{SGPU: Peak GPU memory at the server side} \\
\midrule
\multirow{18}{*}{OPT-125M} 
 & \multirow{3}{*}{SST-2} & ZO-ZO & 85.78\% & 1.38 GB & 3.15 GB \\
 & & FO-FO & 88.54\% & 2.23 GB & 5.43 GB \\
 & & Ours & 87.61\% & 1.38 GB & 5.43 GB \\
\cmidrule{2-6}
 & \multirow{3}{*}{WIC} & ZO-ZO & 53.92\% & 1.38 GB & 4.67 GB \\
 & & FO-FO & 63.79\% & 2.73 GB & 8.12 GB \\
 & & Ours & 62.36\% & 1.38 GB & 8.12 GB \\
\cmidrule{2-6}
 & \multirow{3}{*}{WSC} & ZO-ZO & 56.73\% & 1.38 GB & 2.18 GB \\
 & & FO-FO & 64.42\% & 1.51 GB & 3.46 GB \\
 & & Ours & 63.46\% & 1.38 GB & 3.46 GB \\
\cmidrule{2-6}
 & \multirow{3}{*}{BoolQ} & ZO-ZO & 60.86\% & 2.36 GB & 9.96 GB \\
 & & FO-FO & 62.11\% & 8.67 GB & 21.76 GB \\
 & & Ours & 61.55\% & 2.36 GB & 21.76 GB \\
\cmidrule{2-6}
 & \multirow{3}{*}{CB} & ZO-ZO & 69.64\% & 1.38 GB & 3.15 GB \\
 & & FO-FO & 87.50\% & 2.08 GB & 5.43 GB \\
 & & Ours & 83.93\% & 1.38 GB & 5.43 GB \\
\cmidrule{2-6}
 & \multirow{3}{*}{RTE} & ZO-ZO & 56.69\% & 1.54 GB & 11.80 GB \\
 & & FO-FO & 66.43\% & 4.3 GB & 13.09 GB \\
 & & Ours & 65.31\% & 1.54 GB & 13.09 GB \\
\midrule
\multirow{18}{*}{OPT-1.3B} 
 & \multirow{3}{*}{SST-2}  & ZO-ZO & 90.03\% & 7.25 GB & 7.83 GB \\
 & & FO-FO & 93.92\% & 8.42 GB & 13.92 GB \\
 & & Ours & 92.32\% & 7.23 GB & 13.81 GB \\
\cmidrule{2-6}
 & \multirow{3}{*}{WIC} & ZO-ZO & 56.25\% & 7.23 GB & 7.26 GB \\
 & & FO-FO & 66.51\% & 8.43 GB & 13.98 GB \\
 & & Ours & 62.28\% & 7.23 GB & 13.65 GB \\
\cmidrule{2-6}
 & \multirow{3}{*}{WSC} & ZO-ZO & 56.88\% & 7.23 GB & 7.27 GB \\
 & & FO-FO & 63.77\% & 8.45 GB & 14.16 GB \\
 & & Ours & 62.49\% & 7.23 GB & 13.67 GB\\
\cmidrule{2-6}
 & \multirow{3}{*}{BoolQ} & ZO-ZO & 61.57\% & 7.23 GB & 11.50 GB \\
 & & FO-FO & 70.61\% & 15.41 GB & 45.98 GB \\
 & & Ours & 70.20\% & 7.23 GB & 44.97 GB \\
\cmidrule{2-6}
 & \multirow{3}{*}{CB} & ZO-ZO & 74.52\% & 7.23 GB & 7.84 GB \\
 & & FO-FO & 89.29\% & 9.24 GB & 20.83 GB \\
 & & Ours & 87.58\% & 7.23 GB & 20.49 GB \\
\cmidrule{2-6}
 & \multirow{3}{*}{RTE} & ZO-ZO & 58.73\% & 7.23 GB & 8.12 GB \\
 & & FO-FO & 75.09\% & 9.38 GB & 20.89 GB \\
 & & Ours & 74.28\% & 7.23 GB & 20.89 GB \\
\bottomrule
\end{tabular}
\end{table}

\begin{table}[htbp]
\centering
\caption{Test Accuracy and GPU Memory Comparison of OPT Models (\textbf{LoRA}-Based Parameter-Efficient Fine-Tuning)}
\label{table:LoRA-results}
\begin{tabular}{c|c|c|ccc}
\toprule
\textbf{Model} & \textbf{Task} & \textbf{Method} & \textbf{Test Acc} & \textbf{CGPU} & \textbf{SGPU} \\
\midrule
\multirow{18}{*}{OPT-125M} 
 & \multirow{3}{*}{SST-2} & ZO-ZO & 85.20\% & 0.84 GB & 3.40 GB \\
 & & FO-FO & 87.54\% & 0.86 GB & 8.20 GB \\
 & & Ours & 87.04\% & 0.84 GB & 8.20 GB \\
\cmidrule{2-6}
 & \multirow{3}{*}{WIC} & ZO-ZO & 55.01\% & 0.85 GB & 4.25 GB \\
 & & FO-FO & 66.14\% & 0.8 GB & 10.80 GB \\
 & & Ours & 65.36\% & 0.84
 GB & 10.80 GB \\
\cmidrule{2-6}
 & \multirow{3}{*}{WSC} & ZO-ZO & 56.73\% & 0.88 GB & 5.70 GB \\
 & & FO-FO & 63.16\% & 0.92 GB & 16.00 GB \\
 & & Ours & 62.50\% & 0.88 GB & 16.00 GB \\
\cmidrule{2-6}
 & \multirow{3}{*}{BoolQ} & ZO-ZO & 61.40\% & 2.36 GB & 9.96 GB \\
 & & FO-FO & 62.40\% & 8.67 GB & 33.50 GB \\
 & & Ours & 62.00\% & 2.36 GB & 33.50 GB \\
\cmidrule{2-6}
 & \multirow{3}{*}{CB} & ZO-ZO & 67.86\% & 0.96 GB & 11.20 GB \\
 & & FO-FO & 82.35\% & 1.04 GB & 34.37 GB \\
 & & Ours & 80.30\% & 0.96 GB & 34.37 GB \\
\cmidrule{2-6}
 & \multirow{3}{*}{RTE} & ZO-ZO & 54.15\% & 0.97 GB & 11.8 GB \\
 & & FO-FO & 68.70\% & 1.04 GB & 34.40 GB \\
 & & Ours & 67.31\% & 0.97 GB & 34.40 GB \\
\midrule
\multirow{18}{*}{OPT-1.3B} 
 & \multirow{3}{*}{SST-2} & ZO-ZO & 89.25\% & 5.78 GB & 6.11 GB \\
 & & FO-FO & 93.00\% & 6.06 GB & 7.51 GB \\
 & & Ours & 92.09\% & 5.78 GB & 7.28 GB \\
\cmidrule{2-6}
 & \multirow{3}{*}{WIC} & ZO-ZO & 54.32\% & 5.80 GB & 6.23 GB \\
 & & FO-FO & 58.15\% & 6.15 GB & 7.98 GB \\
 & & Ours & 56.90\% & 5.80 GB & 7.96 GB \\
\cmidrule{2-6}
 & \multirow{3}{*}{WSC} & ZO-ZO & 55.91\% & 5.90 GB & 6.71 GB \\
 & & FO-FO & 60.58\% & 6.48 GB & 9.69 GB \\
 & & Ours & 58.65\% & 5.90 GB & 9.69 GB\\
\cmidrule{2-6}
 & \multirow{3}{*}{BoolQ} & ZO-ZO & 55.80\% & 6.18 GB & 10.52 GB \\
 & & FO-FO & 61.30\% & 12.86 GB & 36.91 GB \\
 & & Ours & 61.10\% & 6.18 GB & 36.91 GB \\
\cmidrule{2-6}
 & \multirow{3}{*}{CB} & ZO-ZO & 64.59\% & 6.15 GB & 7.85 GB \\
 & & FO-FO & 67.26\% & 7.52 GB & 14.91 GB \\
 & & Ours & 66.17\% & 6.15 GB & 14.85 GB \\
\cmidrule{2-6}
 & \multirow{3}{*}{RTE} & ZO-ZO & 54.51\% & 6.12 GB & 7.73 GB \\
 & & FO-FO & 57.40\% & 7.40 GB & 14.32 GB \\
 & & Ours & 57.04\% & 6.12 GB & 14.30 GB \\
\bottomrule
\end{tabular}
\end{table}
\begin{figure}
    \centering
    \includegraphics[width=1\linewidth]{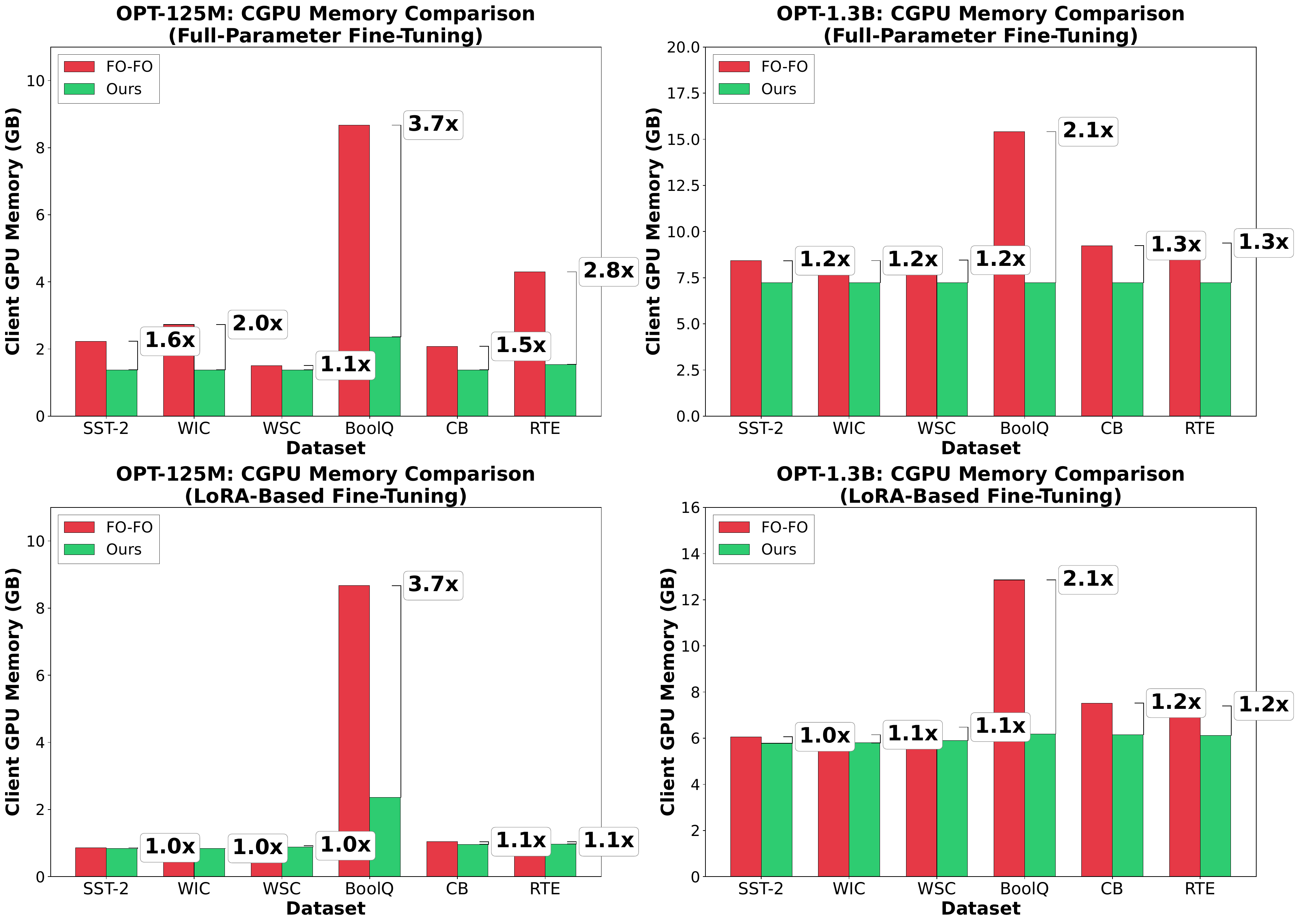}
    \caption{Client GPU (CGPU) Memory Comparison between FO-FO and Our Method for OPT-125M and OPT-1.3B Models}
    \label{fig:cgpu_comparison}
\end{figure}

\section{conclusion}

In this paper, we propose {\alg}, a hybrid-order split learning framework that addresses the memory bottleneck in fine-tuning of LLMs on resource-constrained edge devices. 
By strategically partitioning optimization between ZO method on the client side with FO SGD on the server side, {\alg} achieves significant client-side memory reductions (up to 3.7× compared to the fully FO baseline) while maintaining competitive accuracy within 0.41\%–4.23\% of standard backpropagation-based training.

Our theoretical analysis establishes that {\alg} achieves a convergence rate of $\mathcal{O}(\sqrt{(d_c/TQ)})$, where the rate depends on the client-side parameter dimension $d_c$ rather than the full model dimension $d$. This result demonstrates the benefit of offloading computationally intensive optimization to the server while preserving the memory efficiency of ZO methods on the client. Empirical evaluations across multiple tasks and the OPT model scales validate that {\alg} consistently outperforms the ZO-ZO baseline by up to 15.55\% in accuracy, confirming that server-side FO refinement effectively mitigates the slow convergence typically associated with ZO optimization.

One limitation of {\alg} is the increased number of forward passes required for ZO gradient estimation on the client side, which may increase wall-clock training time despite the memory savings. Additionally, our analysis assumes a single client and does not address data heterogeneity in a federated setting.

\section*{Acknowledgment}
This work is supported in part by RIT CHAI Faculty Seed Grant, NVIDIA Academic Grant Program, NIH awards R16GM159671 and R35GM156653, and NSF grants CNS-2112471 and 2045804. The content is solely the responsibility of the authors and does not necessarily represent the official views of the funding agencies.
% \newpage
\bibliography{reference}
\bibliographystyle{IEEEtran}

\newpage

\section{Appendix}

\subsection{Experiment Setting}
\label{appendix: exp}
\subsubsection{Hyperparameters}
Table~\ref{tab:experiment_settings} reports the hyperparameters used for split learning experiments. For zeroth-order optimization, we fix $\epsilon = 1\text{e-}3$ and use 10 perturbation vectors across all configurations. Both client and server use SGD without momentum. For LoRA fine-tuning, we use rank $r=8$, $\alpha=16$, and apply adapters to query and value projection matrices. We evaluate validation accuracy every 25 communication rounds across all optimizer configurations.

\begin{table*}[htbp]
\centering
\caption{Experiment Settings for Split Learning with Mixed Optimizers}
\label{tab:experiment_settings}
\small
\begin{tabular}{l|l|cccccc}
\toprule
Model & Parameter & SST-2 & CB & WSC & WIC & RTE & BoolQ \\
\midrule
\multirow{4}{*}{OPT-125M+FP} 
 & ZO LR & 1e-6 & 1e-6 & 1e-6 & 1e-6 & 1e-6 & 2e-7 \\
 & FO LR & 1e-3 & 5e-4 & 5e-4 & 1e-3 & 5e-4 & 1e-3 \\
 & Batch Size & 64 & 16 & 64 & 64 & 16 & 16 \\
 & Rounds & 3000 & 1000 & 1000 & 2000 & 1500 & 1500 \\
\midrule
\multirow{4}{*}{OPT-125M+LoRA} 
 & ZO LR & 5e-5 & 5e-5 & 5e-5 & 5e-5 & 1e-6 & 5e-5 \\
 & FO LR & 1e-4 & 1e-4 & 1e-4 & 1e-4 & 5e-3 & 1e-4 \\
 & Batch Size & 64 & 64 & 64 & 64 & 64 & 64 \\
 & Rounds & 2000 & 1000 & 3000 & 2000 & 1500 & 500 \\
\midrule
% \midrule
\multirow{4}{*}{OPT-1.3B+FP} 
 & ZO LR & 1e-6 & 1e-6 & 1e-6 & 1e-6 & 1e-6 & 1e-6 \\
 & FO LR & 1e-4 & 1e-4 & 1e-4 & 1e-4 & 1e-4 & 1e-4 \\
 & Batch Size & 8 & 8 & 8 & 8 & 8 & 8 \\
 & Rounds & 500 & 500 & 500 & 500 & 500 & 500 \\
\midrule
\multirow{4}{*}{OPT-1.3B+LoRA} 
 & ZO LR & 1e-6 & 1e-6 & 1e-6 & 1e-6 & 1e-6 & 1e-6 \\
 & FO LR & 1e-4 & 1e-4 & 1e-4 & 1e-4 & 1e-4 & 1e-4 \\
 & Batch Size & 8 & 8 & 8 & 8 & 8 & 8 \\
 & Rounds & 500 & 500 & 500 & 500 & 500 & 500 \\
\bottomrule
\end{tabular}
\end{table*}

% LoRA vs FT?
% split at multiple layer?

\subsubsection{Prompt} We adopt the prompt templates from \cite{malladi2023fine} without modification. Table~\ref{tab:seq_lengths} reports tokenized sequence lengths on each dataset. 
\begin{table}[ht]
\centering
\caption{Sequence Lengths by Dataset}
\label{tab:seq_lengths}
\begin{tabular}{lrrr}
\toprule
\textbf{Dataset} & \textbf{Mean} & \textbf{Max} & \textbf{Min} \\
\midrule
SST-2  & 27.0   & 64   & 6  \\
WIC   & 43.0   & 82   & 31 \\
WSC   & 60.2   & 87   & 32 \\
BoolQ & 133.1  & 1049 & 23 \\
CB    & 105.6  & 284  & 42 \\
RTE   & 81.5   & 260  & 31 \\
\bottomrule
\end{tabular}
\end{table}

\subsection{GPU Memory Calculation for Split Learning with OPT-125M}
\label{appendix: memory}

\subsubsection{Variable Definitions} We derive theoretical peak GPU memory for split learning with OPT-125M on SST-2, using sequence length $S=64$ and batch size $B=64$. Table~\ref{tab:hyperparameters} defines all variables. Theoretical estimates are validated against empirical measurements from our implementation.

% Wide table for variable definitions
\begin{table}[htbp]
\centering
\caption{Model and training configuration parameters.}
\label{tab:hyperparameters}
\small
\renewcommand{\arraystretch}{1.3}
\begin{tabular}{@{} c l c @{}}
\toprule
\textbf{Symbol} & \textbf{Name} & \textbf{Value} \\
\midrule
$B$ & Batch Size & 64 \\
$S$ & Sequence Length & 64 \\
$H$ & Hidden Dimension & 768 \\
$L$ & Total Layers & 12 \\
$L_c$ & Client Layers & 5 \\
$L_s$ & Server Layers & 7 \\
$A$ & Attention Heads & 12 \\
$d_h$ & Head Dimension ($H/A$) & 64 \\
$d_{ff}$ & FFN Dimension ($4H$) & 3072 \\
$V$ & Vocabulary Size & 50,272 \\
$M$ & Max Positions & 2048 \\
$\beta$ & Bytes/Element (FP32) & 4 \\
$r$ & LoRA Rank & 8 \\
$\alpha$ & LoRA Alpha & 16 \\
\bottomrule
\end{tabular}
\end{table}

\subsubsection{Per-Layer Parameter Count} 
The following formulas compute the parameter count in a single transformer layer. The attention block has four projection matrices (Q, K, V, O), each $H \times H$, plus biases (Eq.~\ref{pattn}). The FFN has two linear layers ($H \rightarrow d_{ff} \rightarrow H$) with biases (Eq.~\ref{pffn}). Two LayerNorms contribute $4H$ parameters via scale $(\gamma)$ and shift $(\beta)$ vectors. The total per-layer count is given in Eq.~\ref{ptotal}.

\noindent\textit{Q,K,V,O projections with biases:}
\begin{equation}
\label{pattn}
P_{\mathrm{attn}} = 4H^2 + 4H
\end{equation}

\noindent\textit{fc1, fc2 with biases:}
\begin{equation}
\label{pffn}
P_{\mathrm{FFN}} = 2Hd_{ff} + d_{ff} + H
\end{equation}

\noindent\textit{2 LayerNorms with $\gamma, \beta$:}
\begin{equation}
\label{pln}
P_{\mathrm{LN}} = 4H
\end{equation}

\noindent\textit{Total per layer:}
\begin{align}
\label{ptotal}
P_{\mathrm{layer}} &= 4H^2 + 2Hd_{ff} + d_{ff} + 9H \notag \\
&= \mathbf{7{,}087{,}872}
\end{align}

\subsubsection{Client Memory (ZO Optimization)}

\paragraph{Parameters.} Model weights for the client's portion of the split model, including token embeddings, positional embeddings, and transformer layers. OPT uses learned positional embeddings with two additional offset positions for padding and BOS tokens, resulting in $(M+2)$ total position embeddings:

\begin{align}
M_{\mathrm{params}}^{\mathrm{c}} &= \bigl(\underbrace{VH}_{\mathrm{tok}} + \underbrace{(M{+}2)H}_{\mathrm{pos}} + \underbrace{L_c P_{\mathrm{layer}}}_{\mathrm{layers}}\bigr) \beta \notag \\
&= 75{,}622{,}656 \times 4 \notag \\
&= \mathbf{288.5~\mathrm{MB}}
\end{align}

\paragraph{Gradients}
ZO estimates gradients via finite differences, so no explicit gradient storage is needed:
\begin{equation}
M_{\mathrm{grad}}^{\mathrm{c}} = \mathbf{0}
\end{equation}

\paragraph{Causal Attention Masks} Lower-triangular masks for autoregressive attention are pre-allocated for the maximum sequence length $M$:

\begin{align}
M_{\mathrm{mask}}^{\mathrm{c}} &= L_c M^2 \beta \notag \\
&= 5 \times 2048^2 \times 4 = \mathbf{80~\mathrm{MB}}
\end{align}

\paragraph{KV Cache} Although the model operates in inference mode without storing activations for backpropagation, the KV cache is required because ZO optimization performs multiple forward passes per update (one per perturbation vector). Caching key and value tensors avoids redundant recomputation across these perturbations:

\begin{align}
M_{\mathrm{KV}}^{\mathrm{c}} &= L_c \cdot 2 \cdot B \cdot A \cdot S \cdot d_h \cdot \beta \notag \\
&= 5 \times 2 \times 64 \times 12 \times 64^2 \times 4 \notag \\
&= \mathbf{120~\mathrm{MB}}
\end{align}

\paragraph{Stored Activations.}
ZO uses inference mode with no autograd graph:
\begin{equation}
M_{\mathrm{act}}^{\mathrm{c}} = \mathbf{0}
\end{equation}

\paragraph{Transient Memory} Peak temporary allocations during a single layer's forward pass, immediately freed after use:

\begin{align}
M_{\mathrm{trans}}^{\mathrm{c}} &= (\underbrace{BSH}_{\text{layer input}} + \underbrace{BSd_{ff}}_{\text{FFN intermediate}} + \underbrace{BSH}_{\text{layer output}}) \beta \notag \\
&= 18{,}874{,}368 \times 4 = \mathbf{72~\mathrm{MB}}
\end{align}

\paragraph{Communication Buffer} The final hidden state tensor returned from the client forward is held in memory until transmission:

\begin{align}
M_{\mathrm{comm}}^{\mathrm{c}} &= BSH \beta \notag = \mathbf{12.6~\mathrm{MB}}
\end{align}

\paragraph{CUDA Runtime Overhead}
We measure the baseline GPU runtime overhead (context initialization, cuDNN/cuBLAS kernels, memory allocator) via \texttt{nvidia-smi} prior to model loading:
\noindent The total runtime overhead is:
\begin{equation}
M_{\mathrm{CUDA}} \approx 400\text{--}800~\mathrm{MB}
\end{equation}

\paragraph{Client Total.}
\begin{equation}
\boxed{M_{\mathrm{total}}^{\mathrm{c}} \approx \mathbf{1.38~\mathrm{GB}}}
\end{equation}
% \noindent{\footnotesize $(302.5 + 0 + 80 + 120 + 0 + 72 + 700)$}

\subsubsection{Server Memory (FO Optimization)}

\paragraph{Parameters}
Model weights for the server's portion, including remaining transformer layers, the language modeling head projection, and final layer normalization:
\begin{align}
M_{\mathrm{params}}^{\mathrm{s}} &= \bigl(\underbrace{VH}_{\mathrm{LM}} + \underbrace{L_s P_{\mathrm{layer}}}_{\mathrm{layers}} + \underbrace{2H}_{\mathrm{LN}}\bigr) \beta \notag \\
&= 88{,}225{,}536 \times 4 \notag \\
&= \mathbf{352.9~\mathrm{MB}}
\end{align}

\paragraph{Gradients}
FO optimization requires storing gradient tensors for all trainable parameters, matching the size of parameter storage:
\begin{align}
M_{\mathrm{grad}}^{\mathrm{s}} &= P_{\mathrm{server}} \cdot \beta \notag \\
&= M_{\mathrm{params}}^{\mathrm{s}} = \mathbf{352.9~\mathrm{MB}}
\end{align}

\paragraph{Causal Attention Masks}
Pre-allocated lower-triangular masks for autoregressive attention across all server layers, sized for maximum sequence length:
\begin{align}
M_{\mathrm{mask}}^{\mathrm{s}} &= L_s M^2 \beta \notag \\
&= 7 \times 2048^2 \times 4 = \mathbf{112~\mathrm{MB}}
\end{align}

\paragraph{Stored Activations}
Backpropagation requires caching intermediate activations from the forward pass. Per layer, we store: layer inputs, QKV projections (for recomputing attention gradients), attention scores and softmax outputs (for attention backward), and FFN intermediate activations:
\begin{align}
M_{\mathrm{act}}^{\mathrm{layer}} &= \bigl(\underbrace{6BSH}_{\mathrm{in/QKV}} + \underbrace{2BAS^2}_{\mathrm{attn}} + \underbrace{BSd_{ff}}_{\mathrm{FFN}}\bigr) \beta \notag \\
&= 144~\mathrm{MB/layer}
\end{align}
\begin{align}
M_{\mathrm{act}}^{\mathrm{s}} &= L_s \times M_{\mathrm{act}}^{\mathrm{layer}} \notag \\
&= 7 \times 144 = \mathbf{1{,}008~\mathrm{MB}}
\end{align}

\paragraph{Logits Peak} \textcolor{red}{(Dominant)}
The language modeling head produces logits over the full vocabulary for each sequence position. Cross-entropy loss computation requires: the logits tensor, shifted versions for causal alignment, softmax probabilities, and their gradients---all at vocabulary scale $V$:
\begin{align}
M_{\mathrm{logits}}^{\mathrm{s}} &= \bigl(\underbrace{BSV}_{\mathrm{logits}} + \underbrace{3B(S{-}1)V}_{\mathrm{shift/soft/grad}}\bigr) \beta \notag \\
&= 813{,}197{,}824 \times 4 \notag \\
&= \mathbf{3.03~\mathrm{GB}}
\end{align}

\paragraph{Transient Memory}
Peak temporary allocations during a single layer's forward pass, reused across layers:
\begin{align}
M_{\mathrm{trans}}^{\mathrm{s}} &= (\underbrace{BSH}_{\text{layer input}} + \underbrace{BSd_{ff}}_{\text{FFN intermediate}} + \underbrace{BSH}_{\text{layer output}}) \beta \notag \\
&= \mathbf{72~\mathrm{MB}}
\end{align}

\paragraph{CUDA Overhead}
Baseline GPU runtime overhead including CUDA context, cuDNN/cuBLAS kernels, and memory allocator structures:
\begin{equation}
M_{\mathrm{CUDA}}^{\mathrm{s}} \approx \mathbf{500~\mathrm{MB}}
\end{equation}

\paragraph{Server Total}
\begin{equation}
\boxed{M_{\mathrm{total}}^{\mathrm{s}} \approx \mathbf{5.43~\mathrm{GB}}}
\end{equation}

\subsection{Ablation Study}
\label{appendix: abl}
We conduct ablation studies on OPT-1.3B with full-parameter fine-tuning using the hyperparameters listed in Table~\ref{tab:experiment_settings}.

\subsubsection{Effect of Split Layer Position}
\textbf{{\alg}'s accuracy-memory trade-off holds across split configurations.}
Table~\ref{tab:split_ablation} compares split positions $k \in \{3, 5, 7\}$ for OPT-1.3B, where the client holds the first $k$ layers and the server handles the remaining $24-k$ layers. Across all six tasks and all three split points, the accuracy of every method decreases gradually as more layers are placed on the client, consistent with the convergence trend predicted by Corollary~IV.4. On the memory side, {\alg} matches ZO-ZO's client GPU usage at each split point, while FO-FO's client memory grows more steeply with $k$.

\begin{table*}[ht]
        \centering
        \caption{Test Accuracy and GPU Memory Comparison of OPT Models at different Split Positions (\textbf{Full-Parameter} Fine-Tuning)}
        \label{tab:split_ablation}
        \small
        \begin{tabular}{ll ccc ccc ccc}
        \toprule
        & & \multicolumn{3}{c}{\textbf{Split $k{=}3$} (Client: 3L, Server: 21L)} 
        & \multicolumn{3}{c}{\textbf{Split $k{=}5$} (Client: 5L, Server: 19L)} 
        & \multicolumn{3}{c}{\textbf{Split $k{=}7$} (Client: 7L, Server: 17L)} \\
        \cmidrule(lr){3-5} \cmidrule(lr){6-8} \cmidrule(lr){9-11}
        \textbf{Task} & \textbf{Method} 
        & \textbf{Acc} & \textbf{CGPU} & \textbf{SGPU} 
        & \textbf{Acc} & \textbf{CGPU} & \textbf{SGPU} 
        & \textbf{Acc} & \textbf{CGPU} & \textbf{SGPU} \\
        \midrule
        
        \multirow{3}{*}{SST-2}
        & ZO-ZO & 90.2\%    & 6.82 & 8.24
                & 90.03\%           & 7.23           & 7.83
                & 89.8\%     & 7.64 & 7.42 \\
        & FO-FO & 94\%    & 7.54 & 14.80
                & 93.92\%           & 8.42           & 13.92
                & 93.84\%     & 9.30 & 13.04 \\
        & Ours  & 92.4\%    & 6.82 & 14.69
                & 92.32\%           & 7.23           & 13.81
                & 91.71\%     & 7.64 & 12.93 \\
        \midrule
        
        \multirow{3}{*}{WIC}
        & ZO-ZO & 56.38\%    & 6.82 & 7.67
                & 56.25\%           & 7.23           & 7.26
                & 55.67\%     & 7.64 & 6.85 \\
        & FO-FO & 66.9\%    & 7.52 & 14.89
                & 66.51\%           & 8.43           & 13.98
                & 65.45\%     & 9.34 & 13.07 \\
        & Ours  & 62.73\%    & 6.82 & 14.56
                & 62.28\%           & 7.23           & 13.65
                & 61.01\%     & 7.64 & 12.74 \\
        \midrule
        
        \multirow{3}{*}{WSC}
        & ZO-ZO & 57.29\%    & 6.82 & 7.68
                & 56.88\%           & 7.23           & 7.27
                & 54.7\%     & 7.64 & 6.86 \\
        & FO-FO & 64\%    & 7.53 & 15.08
                & 63.77\%           & 8.45           & 14.16
                & 63.01\%     & 9.37 & 13.24 \\
        & Ours  & 62.87\%    & 6.82 & 14.59
                & 62.49\%           & 7.23           & 13.67
                & 61.23\%     & 7.64 & 12.75 \\
        \midrule
        
        \multirow{3}{*}{BoolQ}
        & ZO-ZO & 61.91\%    & 6.82  & 11.91
                & 61.57\% & 7.23   & 11.50
                & 60.28\%     & 7.64  & 11.09 \\
        & FO-FO & 71\%    & 11.81 & 49.58
                & 70.61\%  & 15.41 & 45.98
                & 70.2\%     & 19.01 & 42.38 \\
        & Ours  & 70.8\%    & 6.82  & 48.57
                & 70.20\%           & 7.23            & 44.97
                & 69.1\%     & 7.64  & 41.37 \\
        \midrule
        
        \multirow{3}{*}{CB}
        & ZO-ZO & 74.6\% & 6.82 & 8.25
                & 74.52\%  & 7.23  & 7.84
                & 73.86\%   & 7.64 & 7.43 \\
        & FO-FO & 89.5\%    & 8.04 & 22.03
                & 89.29\%           & 9.24           & 20.83
                & 88.1\%     & 10.44& 19.63 \\
        & Ours  & 87.59\%    & 6.82 & 21.69
                & 87.58\%           & 7.23           & 20.49
                & 86.76\%     & 7.64 & 19.29 \\
        \midrule
        
        \multirow{3}{*}{RTE}
        & ZO-ZO & 59.3\%    & 6.82 & 8.53
                & 58.73\%           & 7.23           & 8.12
                & 58.5\%     & 7.64 & 7.71 \\
        & FO-FO & 75.3\%    & 8.13 & 22.14
                & 75.09\%           & 9.38           & 20.89
                & 74.88\%     & 10.63& 19.64 \\
        & Ours  & 74.4\%    & 6.82 & 21.80
                & 74.28\%           & 7.23           & 20.55
                & 74.03\%     & 7.64 & 19.30 \\
        
        \bottomrule
        \end{tabular}
        \end{table*}

\subsubsection{Training Time}
\textbf{{\alg} incurs moderate time overhead compared to FO-FO while remaining faster than ZO-ZO.}
Table~\ref{tab:training_time} reports the wall-clock training time for each optimizer configuration. All times include evaluation every 25 steps; this overhead is identical across methods. ZO-ZO incurs the highest training time due to the multiple perturbation-based forward passes required on both the client and server, whereas FO-FO is the fastest since both sides perform a single forward and backward pass. {\alg} falls between the two, reflecting the overhead of client-side ZO estimation offset by the efficiency of server-side first-order updates.

\begin{table}[h]
\centering
\caption{Training Time (minutes) for 500 Steps on OPT-1.3B 
with Full-Parameter Fine-Tuning}
\label{tab:training_time}
\begin{tabular}{lccc}
\toprule
\textbf{Task} & \textbf{ZO-ZO (min)} & \textbf{FO-FO (min)} & \textbf{Ours (min)} \\
\midrule
SST2  & 54.82   & 29.23  & 35.29  \\
BoolQ & 218.36  & 160.76 & 203.48 \\
CB    & 77.98   & 36.9  & 63.75  \\
WIC   & 66.09   & 39.29  & 58.51  \\
WSC   & 52.32   & 30.19  & 45.23  \\
RTE   & 96.74   & 50.38  & 86.06  \\
\bottomrule
\end{tabular}
\end{table}

\subsubsection{Communication Cost}
\textbf{{\alg} reduces total communication cost by up to 1.9$\times$ compared to FO-FO.}
Table~\ref{tab:comm_1.3b} reports the total communication volume for each optimizer configuration. In the forward direction (Client$\to$Server), all three methods transmit the same volume, since the client always sends intermediate activations of shape $B \times S \times H$. The key difference lies in the backward direction: under ZO-ZO and {\alg}, the server returns only scalar loss values, resulting in negligible backward communication, whereas FO-FO must transmit the full activation gradient tensor back to the client, roughly doubling the total cost.

\begin{table*}[t]
\centering
\caption{Communication Cost Comparison on OPT-1.3B. (\textbf{Full-Parameter} Fine-Tuning)}
\label{tab:comm_1.3b}
\resizebox{\textwidth}{!}{%
\begin{tabular}{l l r r r r}
\toprule
\textbf{Task} & \textbf{Method} & \textbf{Client${\to}$Server (GB)} & \textbf{Server${\to}$Client (GB)} & \textbf{Total (GB)} & \textbf{Per Round (GB)} \\
\midrule
  \multirow{3}{*}{SST-2}
   & ZO-ZO  &  287.00 & 185.32 &   472.32 & 0.94 \\
   & Ours   &  287.00 & 185.32 &   472.32 & 0.94 \\
   & FO-FO  &  287.00 & 461.36 &   748.36 & 1.50 \\
\cmidrule(lr){2-6}
  \multirow{3}{*}{WIC}
   & ZO-ZO  &  410.46 & 215.84 &   626.29 & 1.25 \\
   & Ours   &  410.46 & 215.84 &   626.29 & 1.25 \\
   & FO-FO  &  410.46 & 613.58 & 1,024.04 & 2.05 \\
\cmidrule(lr){2-6}
  \multirow{3}{*}{WSC}
   & ZO-ZO  &  457.95 &  49.28 &   507.23 & 1.01 \\
   & Ours   &  457.95 &  49.28 &   507.23 & 1.01 \\
   & FO-FO  &  457.95 & 504.26 &   962.21 & 1.92 \\
\cmidrule(lr){2-6}
  \multirow{3}{*}{BoolQ}
   & ZO-ZO  & 1,883.62 & 1,047.99 & 2,931.60 & 5.86 \\
   & Ours   & 1,883.62 & 1,047.99 & 2,931.60 & 5.86 \\
   & FO-FO  & 1,883.62 & 2,870.26 & 4,753.88 & 9.51 \\
\cmidrule(lr){2-6}
  \multirow{3}{*}{CB}
   & ZO-ZO  &  871.90 &  69.77 &   941.67 & 1.88 \\
   & Ours   &  871.90 &  69.77 &   941.67 & 1.88 \\
   & FO-FO  &  871.90 & 937.46 & 1,809.36 & 3.62 \\
\cmidrule(lr){2-6}
  \multirow{3}{*}{RTE}
   & ZO-ZO  &  895.47 &  177.59 & 1,073.07 & 2.15 \\
   & Ours   &  895.47 &  177.59 & 1,073.07 & 2.15 \\
   & FO-FO  &  895.47 & 1,062.55 & 1,958.02 & 3.92 \\
\bottomrule
\end{tabular}%
}
\end{table*}

\subsection{Assumptions}
\begin{assum}[$L$-Smooth]\label{assum:smooth}
    The loss function $\gL$ is $L$-smooth for $\forall \theta_1,\theta_2\in \sR^d$:
    \[\|\ \nabla \gL(\theta_1)- \nabla \gL(\theta_2)\|\le L\|\ \theta_1-\theta_2 \|\ ,\]
    thus
    \[ \gL(\theta_2) \geq \gL(\theta_1) + \langle\nabla \gL(\theta_1),\theta_2-\theta_1\rangle +\frac{L}{2} \|\ \theta_2-\theta_1 \|^2\ . \]
\end{assum}
\begin{assum}[Bounded variance]\label{assum:variance}
    The variance of the stochastic gradients w.r.t. the client and server are upper bounded respectively. For $\forall \theta \in \sR^d$, let $\theta=[\theta_c, \theta_s]$, where $\theta_c\in \sR^{d_c}, \theta_s\in \sR^{d_s}$ represent the parameters on client and server, respectively. Then for any mini-batch $\gB \in \gD$:
    \[
    \|\nabla_{\theta_c} \gL(\theta;\gB)-\nabla_{\theta_c} \gL(\theta)\|^2\le \sigma_{c}^2,
    \]
    \[
    \|\nabla_{\theta_s} \gL(\theta;\gB)-\nabla_{\theta_s} \gL(\theta)\|^2\le \sigma_{s}^2.
    \]
    
\end{assum}

\subsection{Technical Lemmas}
\begin{lem}\label{lem:zero}
Let $g(\theta; \mathcal{B})=\widetilde{\nabla} \mathcal{L}(\theta; \mathcal{B})$ be the SPSA gradient estimator as defined in \eqref{eq:spsa_gradient}. We define the smoothed function $\gL_\lambda(\theta) = \mathbb{E}_v[\gL(\theta + \lambda v)]$, where $v$ is uniformly sampled from the Euclidean ball $\sqrt{d} \mathbb{B}^d = \{\theta \in \mathbb{R}^d \mid \|\theta\| \leq \sqrt{d}\}$. The following properties hold:
\begin{enumerate}
    \item $\gL_\lambda(\theta)$ is differentiable and $\mathbb{E}_{u,\gB}[g(\theta; \mathcal{B})] = \nabla \gL_\lambda(\theta)$.
    \item If $\gL(\theta)$ is $L$-smooth, then we have that
    \begin{align}
        \|\nabla \gL(\theta) - \nabla \gL_\lambda(\theta)\| \leq \frac{L}{2} \lambda d^{3/2},\label{eq:3}
    \end{align}
    
    and
    \begin{align}
    \mathbb{E}_u\|g(\theta; \mathcal{B})\|^2 \leq 2d \cdot \|\nabla \gL(\theta)\|^2 + \frac{L^2}{2} \lambda^2 d^3.
    \end{align}
\end{enumerate}
\end{lem}
\begin{rem}
    By \eqref{eq:3} we immediately have 
    \begin{align}
        \|\nabla \gL_\lambda(\theta)\|^2\le 2\|\nabla \gL(\theta)\|^2 +\frac{L^2}{2} \lambda^2 d^3,\\
        \|\nabla \gL(\theta)\|^2\le 2\|\nabla \gL_\lambda(\theta)\|^2 +\frac{L^2}{2} \lambda^2 d^3.
    \end{align}
\end{rem}
\begin{lem}[Bounds on the variance of Zeroth-order Gradient]Let $G(\theta; \mathcal{B}):=\frac{1}{Q}\sum_{q}g^{(q)}(\theta; \mathcal{B})$ be the multi-perturbation zeroth-order gradient as defined in \eqref{eq:multi_pert_gradient}. Under the same condition as Lemma \ref{lem:zero}, we can further bound the variance of the stochastic Zeroth-order Gradient by true gradient at the beginning of the local iteration and the local update distance.
\begin{align*}
\E\| G_c(\theta^t; \mathcal{B}_t)-\nabla_{\theta_c} \gL_{\lambda}(\theta^t)\|^2&\le\frac{2d_c}{Q}\E[\|\nabla_{\theta_c} \gL(\theta^t)\|^2]
+\frac{2d_c\sigma_{c}^2 }{Q}\\
&+\frac{L^2\lambda^2 d_c^3}{2Q}- \frac{1}{Q}\|\nabla_{\theta_c}\gL_{\lambda}(\theta^t)\|^2\numbereq\label{eq:G_var}
\end{align*}
\end{lem}
\begin{proof}\label{lem:var}
First notice that $G_c(\theta^t;\gB_t)=\frac{1}{Q}\sum_{q=1}^{Q}g_{c}^{(q)}(\theta^t;\gB_t)$ and $\E_{u,\gB}[g_{c}^{(q)}(\theta^t;\gB_t)]=\nabla_{\theta_c} \gL_{\lambda}(\theta^t)$.\\
By Lemma \ref{lem:zero}, we have 
\[
\E_u[\|g_{c}^{(q)}(\theta^t;\gB_t)\|^2] \leq 2d_c \cdot \|\nabla_{\theta_c} \gL(\theta^t;\gB_t)\|^2 + \frac{L^2}{2} \lambda^2 d_c^3.
\]
Thus we have
\begin{align*}
    &\E\| G_c(\theta^t; \mathcal{B}_t)-\nabla_{\theta_c} \gL_{\lambda}(\theta^t)\|^2\\
    =&\frac{1}{Q^2}\sum_{p=1}^Q\E\| g_{c}^{(q)}(\theta^t;\gB_t)-\nabla_{\theta_c} \gL_{\lambda}(\theta^t)\|^2\\
    =&\frac{1}{Q^2}\sum_{p=1}^Q\E\| g_{c}^{(q)}(\theta^t;\gB_t)\|^2-\frac{1}{Q}\|\nabla_{\theta_c} \gL_{\lambda}(\theta^t)\|^2\\
    \le&\frac{1}{Q^2}\sum_{p=1}^Q\left[2d_c \E\|\nabla_{\theta_c} \gL(\theta^t;\gB_t)\|^2 + \frac{L^2}{2} \lambda^2 d_c^3\right]\\
    &-\frac{1}{Q}\|\nabla_{\theta_c}\gL_{\lambda}(\theta^t)\|^2 \\
    \le&\frac{1}{Q^2}\sum_{p=1}^Q\left[2d_c (\E \|\nabla_{\theta_c} \gL(\theta^t)\|^2 +\sigma_{c}^2)+ \frac{L^2}{2} \lambda^2 d_c^3\right]\\
    &-\frac{1}{Q}\|\nabla_{\theta_c} \gL_{\lambda}(\theta^t)\|^2\\
    =&\frac{1}{Q}\left[2d_c \E\|\nabla_{\theta_c} \gL(\theta^t)\|^2 +2d_c\sigma_{c}^2+ \frac{L^2}{2} \lambda^2 d_c^3-\|\nabla_{\theta_c} \gL_{\lambda}(\theta^t)\|^2\right] \numbereq\label{eq:7.1.1}
\end{align*}
\end{proof}
\subsection{Proof of algorithm}

\begin{proof}
We now prove the convergence of {\alg} under mixed optimization. The client uses ZO and the server uses FO.
\begin{align*}
    &\E[\gL(\theta^{t+1})-\gL(\theta^{t})]\\
    \le& \underbrace{\E[\langle \nabla_{\theta_s}\gL(\theta^t), \theta_s^{t+1}-\theta_s^t \rangle]}_{\gK_1}+\underbrace{\frac{L}{2}\E[\|\theta_s^{t+1}-\theta_s^{t}\|^2]}_{\gK_2}\\ &+\underbrace{\E[\langle \nabla_{x_{c}}\gL(\theta^t), \theta_{c}^{t+1}-\theta_{c}^t \rangle]}_{\gK_3}+\underbrace{\frac{L}{2}\E[\|\theta_{c}^{t+1}-\theta_{c}^{t}\|^2]}_{\gK_4}\\
\end{align*}

For $\gK_1$:
\begin{align*}
    &\E[\langle \nabla_{\theta_s}\gL(\theta^t), \theta_s^{t+1}-\theta_s^t \rangle]\\
    =&\E[\langle \nabla_{\theta_s}\gL(\theta^t),  -\eta_s G_{s}^{t}(\theta^{t};\gB_{t})\rangle]\\
    =&\E[\langle \nabla_{\theta_s}\gL(\theta^t),  -\eta_s \nabla_{\theta_s} \gL(\theta^{t})\rangle]\\
    =&-\eta_s \|\nabla_{\theta_s}\gL(\theta^t)\|^2
\end{align*}

For $\gK_2$:
\begin{align*}
    &\E[\|\theta_s^{t+1}-\theta_s^{t}\|^2]\\
    =& \eta_s^2\E[\| G_{s}^{t}(\theta^{t};\gB_{t})\|^2]\\
    \le &\eta_s^2(\|\nabla_{\theta_s}\gL(\theta^t)\|^2+\sigma_s^2)\\
\end{align*}
For $\gK_3$:
\begin{align*}
    &\E[\langle \nabla_{x_{c}}\gL(\theta^t), \theta_{c}^{t+1}-\theta_{c}^t \rangle]\\
    =&\E[\langle \nabla_{\theta_c}\gL(\theta^t),  -\eta_c G_{c}^{t}(\theta^{t};\gB_{t})\rangle]\\
    =&\E[\langle \nabla_{\theta_c}\gL(\theta^t),  -\eta_c \left(\nabla_{\theta_c} \gL_{\lambda}^{t}-\nabla_{\theta_c}\gL(\theta^t)+\nabla_{\theta_c}\gL(\theta^t)\right)\rangle]\\
    =&\E[\langle \sqrt{\eta_c}\nabla_{\theta_c}\gL(\theta^t),  -\sqrt{\eta_c}\left( \nabla_{\theta_c} \gL_{\lambda}^{t}- \nabla_{\theta_c}\gL^t\right)\rangle]\\
    &-\eta_c \E[\|\nabla_{\theta_c}\gL(\theta^t)\|^2]\\
    =&\frac{\eta_c}{2}\E\|\nabla_{\theta_c}\gL(\theta^t)\|^2+  \frac{\eta_c}{2}\| \nabla_{\theta_c} \gL_{\lambda}^{t}- \nabla_{\theta_c}\gL^t\|^2\\
    &-\frac{\eta_c}{2}\E\left\| \nabla_{\theta_c} \gL_{\lambda}^{t}\right\|^2-\eta_c \E[\|\nabla_{\theta_c}\gL(\theta^t)\|^2]\\
    \le&-\frac{\eta_c}{2}\E[\|\nabla_{\theta_c}\gL(\theta^t)\|^2]+\frac{\eta_c}{4} L^2\lambda^2 d_c^3-\frac{\eta_c}{2}\E\left\| \nabla_{\theta_c} \gL_{\lambda}^{t}\right\|^2
\end{align*}
For $\gK_4$:
\begin{align*}
    &\E[\|\theta_c^{t+1}-\theta_c^{t}\|^2]\\
    =&\eta_c^2\E\left\| G_{c}^{t}(\theta^{t};\gB_{t})\right\|^2\\
    =&\eta_c^2\E\left\|\nabla_{\theta_c} \gL_{\lambda}^{t}\right\|^2+\eta_c^2\E\left\|G_{c}^{t}-\nabla_{\theta_c} \gL_{\lambda}^{t}\right\|^2\\
\end{align*}
Substituting in \eqref{eq:G_var}, we have
% \begin{align*}
% &\eta_c^2\E\left\|G_{c}^{t}-\nabla_{\theta_c} \gL_{\lambda}^{t}\right\|^2\\
%     \le&\eta_c^2\left(\frac{2d_c}{Q}\E\|\nabla_{\theta_s} \gL(\theta^t)\|^2
% +\frac{2d_c\sigma_{c}^2 }{Q}+\frac{L^2\lambda^2 d_c^3}{2Q} - \frac{1}{Q}\E\|\nabla_{\theta_s}\gL_{\lambda}(\theta^t)\|^2 \right)\\
%     \le&\eta_c^2\left(\frac{2d_c}{Q}\E\|\nabla_{\theta_s} \gL(\theta^t)\|^2
% +\frac{2d_c\sigma_{c}^2 }{Q}+\frac{L^2\lambda^2 d_c^3}{2Q} \right)\\
% \end{align*}

\begin{align*}
&\eta_c^2\E\left\|G_{c}^{t}
- \nabla_{\theta_c}\gL_{\lambda}^{t}\right\|^2\\
    \le&\;\eta_c^2\bigg(
    \frac{2d_c}{Q}\E\|\nabla_{\theta_c} \gL(\theta^t)\|^2\\
    &+\frac{2d_c\sigma_{c}^2 }{Q}
    +\frac{L^2\lambda^2 d_c^3}{2Q} \\
    &- \frac{1}{Q}\E\|\nabla_{\theta_c}
    \gL_{\lambda}(\theta^t)\|^2 \bigg)\\
    \le&\;\eta_c^2\bigg(
    \frac{2d_c}{Q}\E\|\nabla_{\theta_c} \gL(\theta^t)\|^2\\
    &+\frac{2d_c\sigma_{c}^2 }{Q}
    +\frac{L^2\lambda^2 d_c^3}{2Q} \bigg)
\end{align*}
Putting together:
\begin{align*}
    &\E[\gL(\theta^{t+1})-\gL(\theta^{t})]\\
    \le&(\frac{L}{2}\eta_s^2- \eta_s)\E[\|\nabla_{\theta_s}\gL(\theta^t)\|^2]+\frac{L}{2}\eta_s^2\sigma_s^2\\
    -&\frac{\eta_c}{2}\E\|\nabla_{\theta_c}\gL(\theta^t)\|^2+\frac{\eta_c}{4} L^2\lambda^2 d_c^3-\frac{\eta_c}{2}\E\left\| \nabla_{\theta_c} \gL_{\lambda}^{t}\right\|^2\\
    +&\frac{\eta_c^2L}{2}\left(\frac{2d_c}{Q}\E\|\nabla_{\theta_c} \gL(\theta^t)\|^2
    +\frac{2d_c\sigma_{c}^2 }{Q}+\frac{L^2\lambda^2 d_c^3}{2Q} \right)+\frac{\eta_c^2L}{2}\E\left\|\nabla_{\theta_c} \gL_{\lambda}^{t}\right\|^2\\
    \le&(\frac{L}{2}\eta_s^2- \eta_s)\E\|\nabla_{\theta_s}\gL(\theta^t)\|^2+\frac{L}{2}\eta_s^2\sigma_s^2\\
    +&(\frac{\eta_c^2Ld_c}{Q}-\frac{ \eta_c}{2})\E\|\nabla_{\theta_c} \gL(\theta^t)\|^2+\frac{\eta_c^2 L d_c\sigma_{c}^2 }{Q}+\frac{\eta_c L^2\lambda^2 d_c^3}{4Q}\\
    \le &- \frac{\eta_s}{4}\E[\|\nabla_{\theta_s}\gL(\theta^t)\|^2]+\frac{L}{2}\eta_s^2\sigma_s^2 \\
    -&\frac{\eta_c}{4}\E[\|\nabla_{\theta_c} \gL(\theta^t)\|^2]+\frac{\eta_c^2 L d_c\sigma_{c}^2 }{Q}+\frac{\eta_c L^2\lambda^2 d_c^3}{4Q},
\end{align*}
where in the second step we let $\eta_c\ge \frac{1}{L}$, and in the final step we let $\eta_s\le\frac{3}{4L}$ and $\eta_c\le\frac{Q}{4Ld_c}$.

To combine the squared norm of the server gradient $\E[\|\nabla_{\theta_s}\gL\|^2]$ and client gradient $\E[\|\nabla_{\theta_c}\gL\|^2]$, we define the universal step size $\eta:=\eta_s=\eta_c$. Rearranging the terms, we have
\begin{align*}
    &\frac{\eta }{4}\E\|\nabla_{\theta}\gL(\theta^t)\|^2\\
    \le& \E[\gL(\theta^{t})-\gL(\theta^{t+1})]+\frac{L}{2}\eta^2\sigma_s^2+\frac{\eta^2 L d_c\sigma_{c}^2 }{Q}+\frac{\eta L^2\lambda^2 d_c^3}{4Q}
\end{align*}
Taking average from $t=0$ to $T-1$ at both sides:
\begin{align*}
    &\frac{1}{T}\sum_{t=0}^{T}\frac{\eta}{4}\E\|\nabla_{\theta}\gL(\theta^t)\|^2\\
    \le& \frac{1}{T}\E[\gL(\theta^{0})-\gL(\theta^{T})]+\frac{L}{2}\eta^2\sigma_s^2+\frac{\eta^2 L d_c\sigma_{c}^2 }{Q}+\frac{\eta L^2\lambda^2 d_c^3}{4Q}\\
    &\frac{1}{T}\sum_{t=0}^{T}\E[\|\nabla_{\theta}\gL(\theta^t)\|^2]\\
    \le& \frac{4}{\eta T}\E[\gL(\theta^{0})-\gL(\theta^{T})]+2L\eta\sigma_s^2+\frac{4\eta L d_c\sigma_{c}^2 }{Q}+\frac{L^2\lambda^2 d_c^3}{Q}
\end{align*}
Let $\eta=\frac{\sqrt{Q}}{ \sqrt{d_c TL}}$, and take $\lambda^2\le\frac{\sqrt{Q}}{\sqrt{d_c^5T L^3}}$. \\
Thus, we have
\begin{align*}
    &\frac{1}{T}\sum_{t=0}^{T}\E[\|\nabla_{\theta}\gL(\theta^t)\|^2]\\
    \le& \frac{4\sqrt{d_c L}}{\sqrt{TQ}}\E[\gL(\theta^{0})-\gL(\theta^{T})]+\frac{2\sqrt{LQ}}{\sqrt{d_c T}}\sigma_s^2+\frac{4\sqrt{d_cL}\sigma_{c}^2 }{\sqrt{TQ}}+\frac{\sqrt{d_cL}}{\sqrt{TQ}}\\
\end{align*}

\end{proof}

\end{document}